\documentclass[10pt]{article} % For LaTeX2e
\usepackage{graphicx}           % For including images
\usepackage{subcaption}         % Allows for the use of subfigures and subcaptions
\usepackage{url}                % For displaying urls

\usepackage{booktabs}
\usepackage{multicol}
\usepackage{multirow}
\usepackage{wrapfig}
\usepackage[algoruled,lined,linesnumbered]{algorithm2e} 
\usepackage{xcolor}

\usepackage{amsmath,amsfonts,bm,bbm,amsthm}

\def\calS{{\mathcal{S}}}
\def\calA{{\mathcal{A}}}
\def\calR{{\mathcal{R}}}

\def\Ppi{{\mP_{\hspace{-0.7mm}\pi}}}

\def\vzero{{\mathbf{0}}}
\def\vone{{\mathbf{1}}}
\def\vmu{{\mathbf{\mu}}}

\def\vd{{\mathbf{d}}}

\def\vq{{\mathbf{q}}}
\def\vr{{\mathbf{r}}}

\def\vv{{\mathbf{v}}}
\def\vw{{\mathbf{w}}}
\def\vx{{\mathbf{x}}}

\def\vQ{{\mathbf{Q}}}

\def\vmu{{\mathbf{\mu}}}

\def\mP{{\mathbf{P}}}

\newcommand{\E}{\mathbb{E}}

\newtheorem{theorem}{Theorem}%[section]
\newtheorem{lemma}{Lemma}%[section]
\usepackage[accepted]{rlc_v2}

\title{Reward Centering}

\author{\hspace{2mm}Abhishek Naik$^{1,2}$, Yi Wan$^{3}$, Manan Tomar$^{1,2}$, Richard S.~Sutton$^{1,2}$\\
\qquad\texttt{\{abhishek.naik,mtomar,rsutton\}@ualberta.ca, yiwan@meta.com} \\
\hspace{1mm}${^1}$University of Alberta \quad ${^2}$Alberta Machine Intelligence Institute \quad ${^3}$Meta AI
\vspace{-3mm}
}

\begin{document}

\begin{center}
    \maketitle
\end{center}
\vspace{-3mm}

\begin{abstract}
    We show that discounted methods for solving continuing reinforcement learning problems can perform significantly better if they center their rewards by subtracting out the rewards' empirical average.
    The improvement is substantial at commonly used discount factors and increases further as the discount factor approaches one. 
    In addition, we show that if a \textit{problem's} rewards are shifted by a constant, then standard methods perform much worse, whereas methods with reward centering are unaffected.
    % Insight into the benefits of reward centering can be gained from the decomposition of the discounted value function proposed by Blackwell in 1962.
    Estimating the average reward is straightforward in the on-policy setting; we propose a slightly more sophisticated method for %reward centering in 
    the off-policy setting.
    Reward centering is a general idea, so we expect almost every reinforcement-learning algorithm to benefit by the addition of reward centering.
\end{abstract}

%%%%%%%%%%%%%%%%%%%%%%%%%%%%%%%%%%%%%%%%%%%%%%%%%%%%%%%%%%%%%%%%

% introduce the continuing setting and the two formulation 
Reinforcement learning is a computational approach to learning from interaction, where the goal of a learning agent is to obtain as much reward as possible (Sutton \& Barto, 2018). 
In many problems of interest, the stream of interaction between the agent and the environment is \textit{continuing} and cannot be naturally separated into disjoint subsequences or \textit{episodes}.
% The goal of a learning agent is to obtain as much reward as possible.
In continuing problems, agents experience infinitely many rewards, hence a viable way of evaluating performance is to measure the \textit{average reward} obtained per step, or the rate of reward, with equal weight given to immediate and delayed rewards. 
The \textit{discounted-reward} formulation offers another way to interpret a sum of infinite rewards by discounting delayed rewards in favor of immediate rewards. 
The two problem formulations are typically studied separately, each having a set of solution methods or algorithms. 
%\footnote{We specifically mean geometric discounting in which a reward $n$ steps from now is weighted $\gamma^n$ lower than the current reward, where $\gamma\in[0,1)$. 
% The sum of rewards may not be finite with other forms of discounting such as hyperbolic discounting (Fedus et al., 2019), which are applicable only in terminating episodic problems.}
% The formulation is also applicable in episodic problems, where it has been used extensively, including within several impressive demonstrations of RL (e.g., Mnih et al., 2015; Silver et al., 2017; Bellemare et al., 2020; Wurman et al., 2022; Kaufmann et al., 2023).

\begin{figure}[b]
    \centering
    \includegraphics[width=0.8\textwidth]{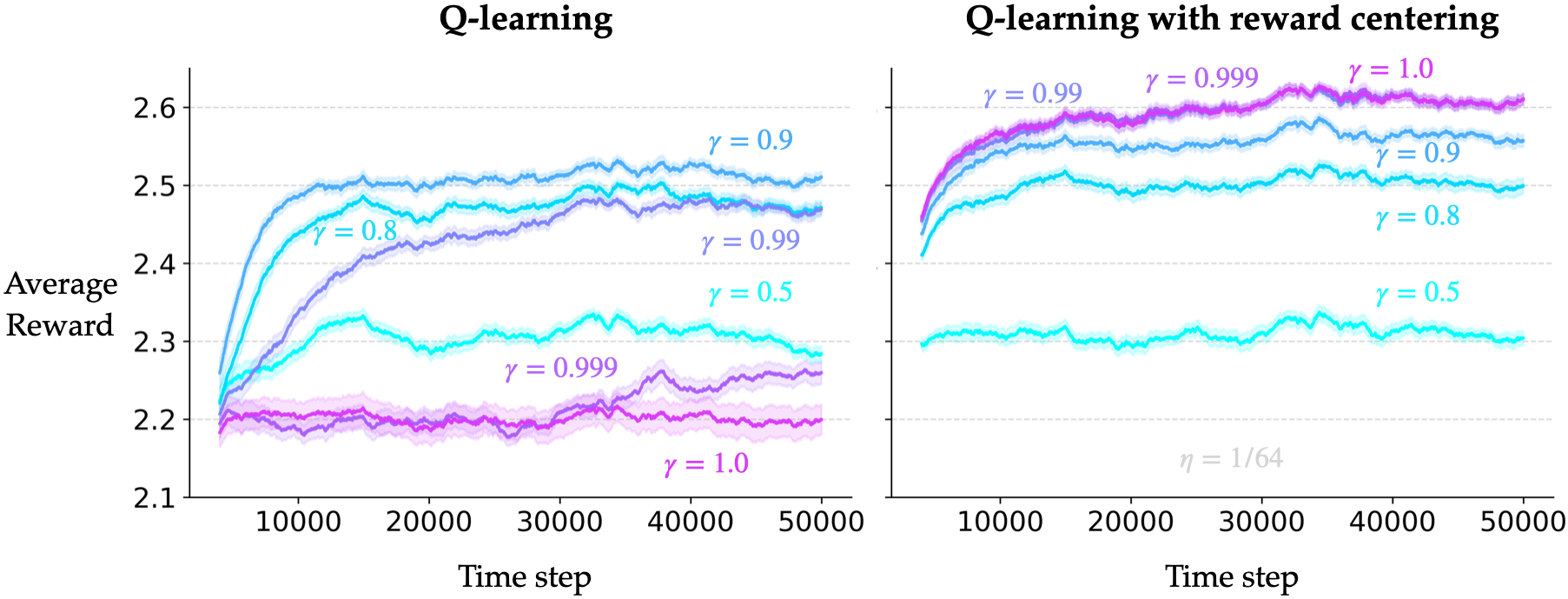}
    \caption{Learning curves showing the difference in performance of Q-learning with and without reward centering for different discount factors on the Access-Control Queuing problem (Sutton \& Barto, 1998). %The $x$-axis denotes the number of time steps of interaction and the $y$-axis denotes the average per-step reward obtained by the agent %over a window of time steps 
    % across 50 runs; the shaded region denotes one standard error. %More details in-text.}
    Plotted is the average per-step reward obtained by the agent across 50 runs w.r.t.~the number of time steps of interaction. 
    The shaded region denotes one standard error. 
    See Section \ref{sec:case_study}.}
    \label{fig:effect_of_gamma}
\end{figure}

In this paper, we show that the simple idea of estimating and subtracting the average reward from the observed rewards can lead to a significant improvement in performance (as in Figure 1) when using common discounted methods such as actor--critic methods (Barto et al., 1983) or Q-learning (Watkins \& Dayan, 1992).
%TD-learning or Q-learning.
% The improvement becomes larger as the discount factor approaches one.
The underlying theory dates back to 1962 with Blackwell's seminal work on dynamic programming in discrete Markov decision processes (MDPs). 
We are still realizing some of its deeper implications, and we discuss the following two in particular:
\begin{enumerate}\itemsep-1mm
    \item Mean-centering the rewards removes a state-independent constant (that scales inversely with $1-\gamma$, where $\gamma$ denotes the discount factor) from the value estimates, enabling the value-function approximator to focus on the relative differences between the states and actions. 
    As a result, values corresponding to discount factors arbitrarily close to one can be estimated relatively easily (e.g., without any degradation in performance; %numerical instability; 
    see Figure \ref{fig:effect_of_gamma}).
    \item Furthermore, mean-centering the rewards (unsurprisingly) makes standard methods robust to any constant offset in the rewards. 
    This can be useful in reinforcement learning applications in which the reward signal is unknown or changing.
    % As a result, the standard initialization of values around zero works well regardless  
\end{enumerate}
\vspace{-2mm}
We begin with \textit{what} reward centering is and \textit{why} it can be beneficial (Section \ref{sec:the_idea}). 
We then show \textit{how} reward centering can be done, starting with the simplest form (within the prediction problem), and show that it can be highly effective when used with discounted-reward temporal difference algorithms (Section \ref{sec:simple_centering}). 
The off-policy setting requires more sophistication; for it we propose another way of reward centering based on recent advances in the \textit{average-reward} formulation for reinforcement learning (Section \ref{sec:value_based_centering}).
Next, we present a case study of using reward centering with Q-learning, in which we (a) propose a convergence result based on recent work by Devraj and Meyn (2021) and (b) showcase consistent trends across a series of control problems that require tabular, linear, and non-linear function approximation (Section \ref{sec:case_study}).
Finally, we discuss the limitations of the proposed methods and propose directions of future work (Section \ref{sec:discussion}).

\section{Theory of Reward Centering}
\label{sec:the_idea}

We formalize the interaction between the agent and the environment by a finite MDP $(\calS, \calA, \calR, p)$, where $\calS$ denotes the set of states, $\calA$ denotes the set of actions, $\calR$ denotes the set of rewards, and $p : \calS \times \calR \times \calS \times \calA \to [0,1]$ denotes the transition dynamics. 
At time step $t$, the agent is in state $S_t \in \calS$, takes action $A_t \in \calA$ using a behavior policy $b:\calA \times \calS \to [0,1]$, observes the next state $S_{t+1} \in \calS$ and reward $R_{t+1} \in \calR$ according to the transition dynamics $p(s', r \mid s, a) = \Pr(S_{t+1}=s', R_{t+1}=r \mid S_t=s, A_t=a)$.
We consider continuing problems, where the agent-environment interaction goes on \textit{ad infinitum}.
The agent's goal is to maximize the average reward obtained over a long time (formally defined in \eqref{eq:def_avg_reward_and_differential_value_fn}).
We consider methods that try to achieve this goal by estimating the expected discounted sum of rewards from each state for $\gamma\in[0,1)$: $v_\pi^\gamma(s) \doteq \E [ \sum_{t=0}^\infty \gamma^{t} R_{t+1} \mid S_t=s, A_{t:\infty}\sim\pi ], \forall s$. 
Here, the discount factor is not part of the problem but an algorithm parameter (see Naik et al.~(2019) or Sutton \& Barto's (2018) Section 10.4 for an extended discussion on objectives for continuing problems).
%\footnote{Naik et al.~(2019) and Sutton \& Barto's (2018) Section 10.4 show the discounted-reward problem is not well-defined for control in the continuing setting with function approximation. 
% Our theoretical analysis is restricted to the \textit{tabular} case for which there is a well-defined discounted-reward objective.
% However, we keep this caveat in mind when reporting results of discounted solution methods with function approximation later in this paper.}

% obtained by the agent.
% Further, we consider methods that estimate the expected discounted sum of rewards from each state for $\gamma\in[0,1)$: $v_\pi^\gamma(s) \doteq \E [ \sum_{t=0}^\infty \gamma^{t} R_{t+1} \mid S_t=s, A_{t:\infty}\sim\pi ], \forall s$, and evaluate the methods based on long-term undiscounted reward.\footnote{Naik et al.~(2019) and Sutton \& Barto's (2018) Section 10.4 show the discounted-reward problem is not well-defined for control in the continuing setting with function approximation.% obtained by the agent. 

% In its simplest form, reward centering is almost trivial: subtract the average reward from the observed rewards. 
Reward centering is a simple idea: subtract the empirical average of the observed rewards from the rewards.
Doing so makes the modified rewards appear \textit{mean centered}.
% The average reward can be estimated empirically.
% The core idea is not new. 
The effect of mean-centered rewards is well known in the bandit setting.
For instance, Sutton and Barto (2018, Section 2.8) demonstrated that estimating and subtracting the average reward from the observed rewards can significantly improve the rate of learning. 
Here we show that the benefits extend to the full reinforcement learning problem and are magnified as the discount factor $\gamma$ approaches one% (bandit problems correspond to $\gamma=0$)
.

% The core idea behind centering is revealed by the Laurent-series decomposition of the \textit{discounted} value function. 
% Blackwell (1962: Theorem 4a) showed that—in the tabular case—the discounted value function $v_\pi^\gamma: \calS \to \mathbb{R}$ for a policy $\pi$ corresponding to a discount factor $\gamma$ can be decomposed as:
The reason underlying the benefits of reward centering is revealed by the Laurent-series decomposition of the discounted value function. 
The discounted value function can be decomposed into two parts, one of which is a constant that does not depend on states or actions and hence is not involved in, say, action selection.
% In the tabular case, the discounted value function $v_\pi^\gamma: \calS \to \mathbb{R}$ for a policy $\pi$ corresponding to a discount factor $\gamma$ can be decomposed into two parts: a constant that does not depend on states or actions, and a part that does: \blue{lead with Blackwell and the two components}
Mathematically, for the tabular discounted value function $v_\pi^\gamma: \calS \to \mathbb{R}$ of a policy $\pi$ corresponding to a discount factor $\gamma$: 
\begin{align}
\label{eq:laurent_series_expansion_state_values}
    v_\pi^\gamma(s) = \frac{r(\pi)}{1-\gamma} + \tilde{v}_\pi(s) + e_\pi^\gamma(s), \qquad \forall s,
\end{align} 
where $r(\pi)$ is the state-independent average reward obtained by policy $\pi$ and $\tilde{v}_\pi(s)$ is the \textit{differential} value of state $s$, each defined for %aperiodic unichain 
ergodic MDPs (for ease of exposition) as (e.g., Wan et al., 2021):
% \small
\begin{align}
\vspace{-0.5mm}
    {\small r(\pi) \doteq \lim_{n\to\infty} \frac{1}{n} \sum_{t=1}^n \E \big[ R_{t} \mid S_0, A_{0:t-1}\sim\pi \big],} \quad 
    {\small \tilde{v}_\pi(s) \doteq \E \left[ \sum_{k=1}^\infty \big( R_{t+k} - r(\pi) \big) \mid S_t=s, A_{t:\infty}\sim\pi \right]}, \label{eq:def_avg_reward_and_differential_value_fn}
\vspace{-0.5mm}
\end{align}
% \normalsize
and $e_\pi^\gamma(s)$ denotes an error term that goes to zero as the discount factor goes to one (Blackwell, 1962: Theorem 4a; also see Puterman's (1994) Corollary 8.2.4). 
This decomposition of the state values also implies a similar decomposition for state--action values.
% \begin{align}
% \label{eq:laurent_series_expansion_action_values}
%     q_\pi^\gamma(s,a) = \frac{r(\pi)}{1-\gamma} + \tilde{q}_\pi(s,a) + e_\pi^\gamma(s,a), \qquad \forall s,a.
% \end{align} 

% The Laurent-series decomposition shows that the discounted value function comprises of a constant state(--action)-independent term.
% This explains the results in Figure \ref{fig:bandit_centering}. 
% This explains the improvements in the bandit setting (see Sutton \& Barto's (2018) Figure 2.5).
The Laurent-series decomposition explains how reward centering can help learning in bandit problems such as the one in Sutton \& Barto's (2018) Figure 2.5.
% The first term in the decomposition \eqref{eq:laurent_series_expansion_state_values} explains the improvements in the bandit setting (see Sutton \& Barto's (2018) Figure 2.5), 
There, the action-value estimates %(of the single state) 
are initialized to zero and the true values are centered around $+4$.
The actions are selected based on their \textit{relative} values, but each action-value estimate must independently learn the same constant offset.
Approximation errors in estimating the offset can easily mask the relative differences in actions, especially if the offset is large. 
% This is especially true when we consider the full RL problem with a non-zero discount factor because the offset scales inversely with $(1-\gamma)$.

%\def\A{s_A}
\def\A{\text{A}}
\def\B{\text{B}}
\def\C{\text{C}}
In the full reinforcement learning problem, the state-independent offset can be quite large. 
For example, %let us start the prediction setting.
consider the three-state Markov reward process shown Figure \ref{fig:3stateMRP_centering} (induced by some policy $\pi$ in some MDP). 
The reward is $+3$ on transition from state $\A$ to state $\B$, and $0$ otherwise.
The average reward is $r(\pi)=1$.
The discounted state values for three discount factors are shown in the table.
Note the magnitude of the standard discounted values and especially the jump when the discount factor is increased.
Now consider the discounted values with the constant offset subtracted from each state, $v_\pi^\gamma(s) - r(\pi)/(1-\gamma)$, which we call the \textit{centered discounted values}.
The centered values are much smaller in magnitude and change only slightly when the discount factor is increased.
The differential values are also shown for reference.
These trends hold in general: for any problem, the magnitude of the discounted values increase dramatically as the discount factor approaches one whereas the centered discounted values change little and approach the differential values. %\blue{Figure 2 can appear here}

\begin{figure}[h]
\vspace{-2mm}
\centering
    \subfloat{
        \centering
        \includegraphics[width=0.3\textwidth]{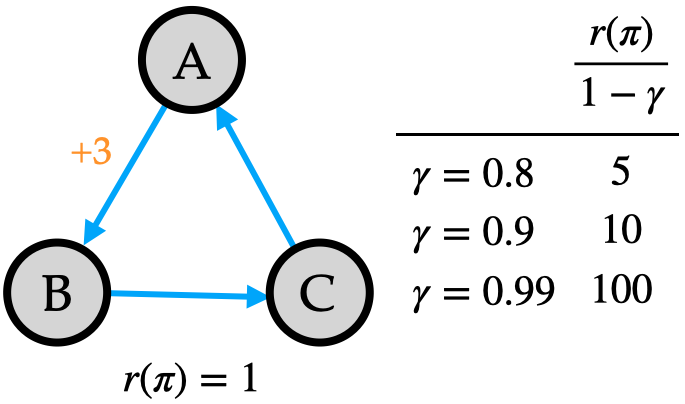}
    }%
    \quad
    \subfloat{\footnotesize
        \begin{tabular}{@{}clccc@{}}
        % \toprule
         State &  & $\A$ & $\B$ & $\C$ \\ \midrule
        \multirow{3}{*}{\begin{tabular}[c]{@{}c@{}}Standard\\ discounted values\end{tabular}} & $\gamma=0.8$ & 6.15 & 3.93 & 4.92 \\
         & $\gamma=0.9$ & 11.07 & 8.97 & 9.96 \\
         & $\gamma=0.99$ & 101.01 & 98.99 & 99.99 \\ \midrule
        \multirow{3}{*}{\begin{tabular}[c]{@{}c@{}}Centered\\ discounted values\end{tabular}} & $\gamma=0.8$ & 1.15 & -1.07 & -0.08 \\
         & $\gamma=0.9$ & 1.07 & -1.03 & -0.04 \\
         & $\gamma=0.99$ & 1.01 & -1.01 & -0.01 \\ \midrule
        Differential values &  & 1 & -1 & 0 \\ \bottomrule
        \end{tabular}
    }
    \vspace{-1mm}
    \caption{Comparison of the standard and the centered discounted values on a simple example. %\red{can add a final column with the average value}
    %values are rounded off to two decimal places, hence the centered values for 0.99 don't sum up to zero
    }
    \label{fig:3stateMRP_centering}
    \vspace{-1mm}
\end{figure}

Formally, the centered discounted values are the expected discounted sum of mean-centered rewards: 
\vspace{-5mm}
\begin{align}
\label{eq:def_centered_state_value_function}
    \tilde v_\pi^\gamma(s) \doteq 
    \E \left[ \sum_{t=0}^\infty \gamma^{t} \big( R_{t+1} - r(\pi) \big) \mid S_t=s, A_{t:\infty}\sim\pi \right], \quad 
    v_\pi^\gamma(s) = \frac{r(\pi)}{1-\gamma} + {\color{gray}\overbrace{{\color{black}\tilde{v}_\pi(s) + e_\pi^\gamma(s)}}^{ \tilde v_\pi^\gamma(s)}},\ \ \forall s,
\end{align}
where $\gamma \in [0,1]$. 
When $\gamma=1$, the centered discounted values are the same as the differential values, that is, $\tilde v_\pi^\gamma(s) = \tilde v_\pi(s), \forall s$%; when $r(\pi)=0, \tilde v_\pi^\gamma(s) = v_\pi^\gamma(s), \forall s$
.
More generally, the centered discounted values are the differential values plus the error terms from the Laurent-series decomposition, as shown on the right above.
% \begin{align*}
    % v_\pi^\gamma(s) = \frac{r(\pi)}{1-\gamma} + \underbrace{\tilde{v}_\pi(s) + e_\pi^\gamma(s)}_{\color{gray}\textstyle \tilde v_\pi^\gamma(s)}, \qquad \forall s,
% \end{align*} 

Reward centering thus enables capturing all the information within the discounted value function via two components: (1) the constant average reward and (2) the centered discounted value function.
Such a decomposition can be immensely valuable: 
(a) As $\gamma\to 1$, the discounted values tend to explode but the centered discounted values remain small and tractable. 
% As a result, estimating the centered values may yield better sample-complexity bounds than methods that estimate the uncentered values (e.g., the bounds for Q-learning involve powers of $1/(1-\gamma)$ (Qu \& Wierman, 2020; Wainwright, 2019; Even-Dar et al., 2003)). 
(b) If the \textit{problems}' rewards are shifted by a constant $c$, then the magnitude of the discounted values increases by $c/(1-\gamma)$, but the centered discounted values are unchanged because the average reward increases by $c$. 
These effects are demonstrated in the following sections.

Reward centering also enables the design of algorithms in which the discount factor (an algorithm parameter) can be changed within the lifetime of a learning agent.
This is usually inefficient or ineffective with standard discounted algorithms because their uncentered values can change massively (Figure \ref{fig:3stateMRP_centering}).
In contrast, centered values may change little, and the changes become minuscule as the discount factor approaches 1.
We discuss this exciting direction in the final section.

% The precursor to these potential benefits is the estimation of the average reward from data.
To obtain these potential benefits, we need to estimate the average reward from data.
In the next section we show that even the simplest method can be quite effective. %take us quite far.

\section{Simple Reward Centering}
\label{sec:simple_centering}

The simplest way to estimate the average reward is to maintain a running average of the rewards observed so far.
That is, if $\bar{R}_t \in \mathbb{R}$ denotes the estimate of the average reward after $t$ time steps, then $\bar{R}_t = \sum_{k=1}^t R_{k}$. 
More generally, the estimate can be updated with a step-size parameter $\beta_t$:
\begin{align}
\label{eq:update_centeredTD_rbar_simple}
    \bar{R}_{t+1} \doteq \bar{R}_t + \beta_t (R_{t+1} - \bar{R}_t). 
\end{align}
This update leads to an unbiased estimate of the average reward $\bar R_t\approx r(\pi)$, for the policy $\pi$ generating the data, if the step sizes follow standard conditions (Robbins \& Monro, 1951). %\blue{mention Robbins-Monro; highlight sample average}.
% We refer to learning  simple reward centering.

Simple centering \eqref{eq:update_centeredTD_rbar_simple} can be used with almost any reinforcement learning algorithm.
For example, it can be combined with conventional temporal-difference (TD) learning (see Sutton, 1988a) to learn a state-value function estimate $\tilde V^\gamma:\calS \to \mathbb{R}$ by updating, 
on transition from $t$ to $t+1$: %the value estimates $\tilde V^\gamma:\calS \to \mathbb{R}$ are updated by:
\begin{align}
\label{eq:update_centeredTD_values}
   \tilde{V}^\gamma_{t+1}(S_t) \doteq \tilde{V}^\gamma_t(S_t) + \alpha_t \big[ (R_{t+1} - \bar{R}_t) + \gamma \tilde{V}_t^\gamma(S_{t+1}) - \tilde{V}_t^\gamma(S_{t}) \big], 
    % \tilde{V}^\gamma(S_t) \leftarrow \tilde{V}^\gamma(S_t) + \alpha \big[ (R_{t+1} - \bar{R}_t) + \gamma \tilde{V}^\gamma(S_{t+1}) - \tilde{V}^\gamma(S_{t}) \big], 
\end{align}
with $\tilde{V}^\gamma_{t+1}(s) \doteq \tilde{V}^\gamma_{t}(s), \forall s \neq S_t$, where $\alpha_t>0$ is a step-size parameter.

We used four algorithmic variations of \eqref{eq:update_centeredTD_values} differing only in the definition of $\bar R_t$ in our first set of experiments.
One algorithm used $\bar{R}_t = 0, \forall t$, and thus involves no reward centering.
The second algorithm used the best possible estimate of the average reward: $\bar{R}_t = r(\pi), \forall t$; we call this \textit{oracle centering}. 
%That is, the rewards are centered using the true average reward of the target policy obtained from an oracle.
The third algorithm used \textit{simple reward centering} as in \eqref{eq:update_centeredTD_rbar_simple}. 
The fourth algorithm used a more sophisticated kind of reward centering which we discuss in the next section.

%the values of the other states remain unchanged.
% We call \eqref{eq:update_centeredTD_rbar_simple} and \eqref{eq:update_centeredTD_values} \textit{TD-learning with simple reward centering}. 
%As we demonstrate shortly, this simple approach to reward centering works quite well.
% The following investigative experiment shows that (a) learning the average reward and the value estimates separately indeed increases higher rate of learning, and (b) that the benefits become larger as the discount factor approached one. 
The environment was an MDP with seven states in a row with two actions in each state. 
The right action from the rightmost state leads to the middle state with a reward of $+7$ and the left action from the leftmost state leads to the middle state with a reward of $+1$; all other transitions have zero rewards.
The target policy takes both actions in each state with equal probability, that is, $\pi(\text{left}|\cdot) = \pi(\text{right}|\cdot) = 0.5$.
The average reward corresponding to this policy is $r(\pi)=0.25$.

% We tested three variants of the TD-learning algorithm: 1) TD-learning with simple reward centering, 2) TD-learning without reward centering ($\bar R_t\doteq 0, \forall t$), and 3) TD-learning with rewards that are centered by an oracle.
%rewards that are centered by the simple running estimate of the average reward.
%the aforementioned simple reward centering. %reward centering using \eqref{eq:update_centeredTD_rbar_simple}, which we call the simple approach. 

Our first experiment applied the four algorithms to the seven-state MDP with two discount factors, $\gamma=0.9$ and $0.99$. 
All algorithms were run with a range of values for the step-size parameters $\alpha$.
The algorithms that learned to center were run with different values of $\eta$, where $\beta = \eta \alpha$ (without loss of generality).
Each parameter setting for each algorithm was run for 50,000 time steps, and then repeated for 50 runs.
The full experimental details are in Appendix \ref{app:more_results}.
As a measure of performance at time $t$, we used the root-mean-squared value error (RMSVE; see Sutton \& Barto, 2018, Section 9.2) between $\tilde{V}^\gamma_t$ and $\tilde{v}_\pi^\gamma$ for the centered algorithms, and between $\tilde{V}^\gamma_t$ and $v_\pi^\gamma$ for the algorithm without centering. 
There was no separate training and testing period.
% In both
% w.r.t.\ the steady-state distribution under $\pi$ 
% The centering and oracle-centered methods estimate the centered discounted value function $\tilde{\vv}^\gamma_\pi$, so for an apples-to-apples comparison, we added $\bar{R}_t/(1-\gamma)$ to the centered estimates to compute the uncentered value estimates at each time step. %???

% \begin{figure}[!b]
%     \centering
%     \includegraphics[width=0.7\linewidth]{figures/RW7_onpolicy_simple_0.9_0.99.png}
%     \caption{Learning curves demonstrating the performance of standard TD-learning (blue), TD-learning with rewards that are centered by an oracle (orange), and TD-learning with centering (green) on a random-walk problem. 
%     Note the difference in the scales of two axes.}
%     \label{fig:centering_results_prediction_onpolicy}
% \end{figure}

Learning curves for this experiment and each value of $\gamma$ are shown in the first column of Figure \ref{fig:centering_results_prediction_offpolicy}.
For all algorithms, we show only curves for the $\alpha$ value that was best for
TD-learning without reward centering. 
For the centering methods, the curve shown is for the best choice of $\eta$ %, where $\beta_t = \eta \alpha_t$ (without loss of generality), 
from a coarse search over a broad range. 
%\blue{mention choice of parameters and $\beta=\eta\alpha$}
Each solid point represents the RMSVE averaged over the 50 independent runs; the shaded region shows one standard error.
% All the experiments in this paper follow this general experimental setup; more details (including number of runs, range of parameter values tested) are in Appendix \ref{app:more_results}.

% Let us first consider the plot on the left corresponding to $\gamma=0.9$. 
% For the uncentered and the centering approach, the learning curve starts at just over 2.5. 
% One can verify this by computing the true discounted values for $\gamma=0.9$.
% Alternatively, we can get an quick estimate thanks to the Laurent-series decomposition \eqref{eq:laurent_series_expansion_state_values} which says that all the values have a state-independent constant of $r(\pi)/(1-\gamma)$—in this case, $0.25/(1-0.9) = 2.5$.
% In the plot on the left, the oracle-centered learning curve starts much lower than the other two because it magically has a fixed average-reward estimate of $0.25$ from the start. 
First note that the learning curves start much lower when the rewards are centered by an oracle; for the other algorithms, the first error is of the order $r(\pi)/(1-\gamma)$. 
TD-learning without centering (blue) eventually reached the same error rate as the oracle-centered algorithm (orange), as expected. 
Learning the average reward and subtracting it (green) indeed helps reduce the RMSVE much faster compared to when there is no centering.
However, the eventual error rate is slightly higher, which is expected because the average-reward estimate is changing over time, leading to more variance in the updates compared to the uncentered or oracle-centered version. 
Similar trends hold for the larger discount factor (lower left), with the uncentered approach appearing much slower in comparison (note the difference in axes' scales).
In both cases, we verified that the average-reward estimate across the runs was around $0.25$.
% One important observation from plot on the right is when $\eta$ is larger, the RMSVE reduction is higher in the beginning; however, this also results in a larger asymptotic error rate.
% This suggests the use of a step-size adaptation technique for the average-reward estimate, which sets a larger step size when the errors are large and prevalent and scales them down otherwise.
% We do not explore step-size adaptation in this chapter; this is an appealing direction for future work.
% However, we verified that the reward-centering approaches indeed learn an average-reward estimate that is around $0.25$ on average in both experiments.

\begin{figure}[!t]
    \vspace{-3mm}
    \centering
    \includegraphics[width=0.9\linewidth]{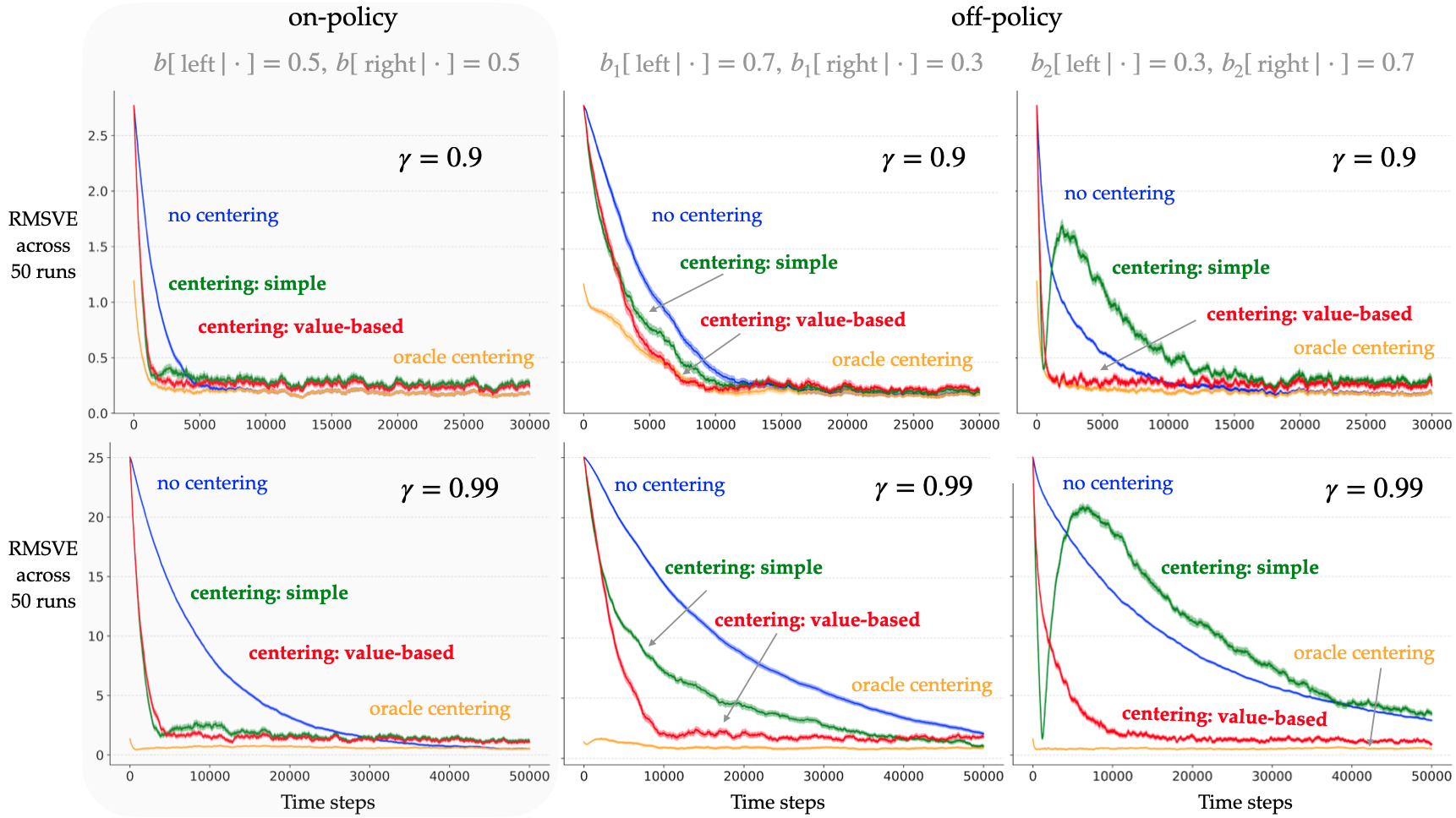}
    \vspace{-3mm}
    \caption{Learning curves demonstrating the performance of TD-learning with and without reward centering on one on-policy problem and two off-policy problems.
    %standard TD-learning (blue), TD-learning with rewards that are centered by an oracle (orange), TD-learning with simple centering (green), and TD-learning with value-based centering (red) on on- and off-policy problems for two discount factors.
    }
    \label{fig:centering_results_prediction_offpolicy}
    \vspace{-3mm}
\end{figure}

These experiments show that the simple reward-centering technique can be quite effective in the on-policy setting, and the effect is more pronounced for larger discount factors.

\textbf{Limitations in the Off-policy Setting}: 
\eqref{eq:update_centeredTD_rbar_simple} leads to an unbiased estimate of the behavior policy's average reward, which means that in the \textit{off-policy} setting the average-reward estimate $\bar{R}$ will converge to $r(b)$, not $r(\pi)$.
% However, in the off-policy setting, the behavior policy $b$ is different from the target policy $\pi$, which implies that $\bar{R}$ will converge to $r(b)$, not to $r(\pi)$. 
Adding an importance-sampling ratio to the update is not enough to guarantee convergence to $r(\pi)$ because importance sampling only corrects the mismatch in action distributions, % (that is, the policies), 
not the mismatch in the resulting state distributions. 

% Note that an incorrect average-reward estimate is not a deal-breaker: 
% standard algorithms that do not center the rewards can be perceived as using a fixed inaccurate estimate of the average reward (zero), yet in the tabular case they are guaranteed to converge to the values of the target policy.
% So convergence is not an issue, the rate of learning is. 

% As a result, estimating the centered values may yield better sample-complexity bounds than methods that estimate the uncentered values (e.g., the bounds for Q-learning involve powers of $1/(1-\gamma)$ (Qu \& Wierman, 2020; Wainwright, 2019; Even-Dar et al., 2003)). 

Let us consider the effect of an inaccurate estimate of the average reward. 
First, note that the centered discounted value function also satisfies a recursive Bellman equation:
\begin{align}
    \tilde{v}^\gamma(s) = \sum_a \pi(a|s) \sum_{s',r}p(s',r\mid s,a) \big[ r - \bar{r} + \gamma \tilde{v}^\gamma(s') \big], \quad
    \text{or,}\quad \tilde{\vv}^\gamma = \vr_\pi - \bar{r}\vone + \gamma \Ppi \tilde{\vv}^\gamma, \label{eq:bellman_equation_centered_evaluation_state}
\end{align}
where, $\tilde{\vv}^\gamma$ denotes a vector in $\mathbb{R}^{|\calS|}$, $\vr_\pi$ is the vector of the expected one-step reward from each state, $\bar{r}$ is a scalar variable, $\vone$ is a vector of all ones, and $\Ppi$ is the state-to-state transition matrix induced by the policy $\pi$. 
It is easy to verify that the solution tuples $(\tilde{\vv}^\gamma, \bar{r})$ of \eqref{eq:bellman_equation_centered_evaluation_state} are of the form $\big(\tilde{\vv}_\pi^\gamma + c\vone, r(\pi) - c(1-\gamma)\big), \forall c\in \mathbb{R}$, where $\tilde{\vv}_\pi^\gamma$ denotes the centered differential value function \eqref{eq:def_centered_state_value_function} corresponding to policy $\pi$ and discount factor $\gamma$. 
Equivalently, we can write the family of solutions as $\big(\tilde{\vv}_\pi^\gamma + \frac{k}{1-\gamma}\vone, r(\pi) - k\big), \forall k\in \mathbb{R}$, which shows that if the average-reward estimate is off by $k$, then the centered discounted values each have a constant offset of $k/(1-\gamma)$. 
This is undesirable. 
The primary motivation of reward centering is to eliminate the potentially large offset from the estimates. %\footnote{
% Consider the case when the average-reward estimate is zero, that is $k=r(\pi)$. 
% The corresponding solution of the value estimates is then $\tilde{\vv}_\pi^\gamma + {r(\pi)}/{(1-\gamma)}$, which is the standard (uncentered) discounted value function $\vv_\pi^\gamma$.} 
So we desire a way to estimate the target policy's average reward while behaving according to a different behavior policy.

However, note that an inaccurate estimation of the average reward is not a deal-breaker: standard algorithms that do not center the rewards can be perceived as using a fixed inaccurate estimate of the average reward (zero), yet they are guaranteed to converge to the true values of the target policy in the tabular case.
So the issue is less about convergence and more about the rate of learning. 
Estimating the average reward accurately may yield better sample-complexity bounds when using standard methods than simply estimating the uncentered values (e.g., the bounds for Q-learning involve powers of $1/(1-\gamma)$ (Qu \& Wierman, 2020; Wainwright, 2019; Even-Dar et al., 2003)). 
We also saw in Figure \ref{fig:centering_results_prediction_offpolicy} that when the rewards are centered by an oracle, the rate of learning is much higher compared to when there is no centering.
% If $\bar{R}=r(\pi)$ is the ideal value that results in the highest rate of learning (relatively), then $\bar{R}\in\big(0, 2r(\pi)\big)$ should roughly result in a higher rate of learning than the baseline rate of uncentered algorithm that sets $\bar{R}=0$ (assuming $r(\pi)>0$ without loss of generality). 
% Of course, the average reward is estimated from data, which might nullify increases in the rate of learning at the edges of above boundaries. 

% In summary, the simple method of reward centering \eqref{eq:update_centeredTD_rbar_simple} can be highly effective when the average reward of the behavior policy is close to that of the target policy. 
In summary, the effectiveness of reward centering increases with the accuracy of the average-reward estimate.
Thus, even the simple method of reward centering \eqref{eq:update_centeredTD_rbar_simple} can be highly effective when the average reward of the behavior policy is close to that of the target policy. 
This may be true when the two policies are similar, like a greedy target policy and an $\epsilon$-greedy behavior policy with a relatively small value of $\epsilon$. 
However, the benefits of reward centering in terms of rate of learning may reduce and even disappear as the difference in the two policies increases.
In the following section, we present a subtly advanced approach to estimate the average reward more accurately in the off-policy setting.

%%%%%%%%%%%%%%%%%%%%%%%%%%%%%%%%%%%%%%%%%%%%%%%%%%%%%%%%%%%%%%%%

%\pdfbookmark[0]{Value-based Reward Centering}{}
\section{Value-based Reward Centering}
\label{sec:value_based_centering}

We drew inspiration from the \textit{average-reward} formulation of reinforcement learning, where estimating the average reward in the off-policy setting is a pertinent problem.
In particular, Wan et al.~(2021) recently showed that using the \textit{temporal-difference} (TD) error (instead of the \textit{conventional} error in \eqref{eq:update_centeredTD_rbar_simple}) leads to an unbiased estimate of the reward rate in the tabular off-policy setting.
It turns out that this idea from the average-reward formulation is quite effective even in the discounted-reward formulation, which is the focus of this paper.
We show that if the behavior policy takes all the actions that the target policy does (the exact distribution over actions may differ arbitrarily), then we get a good approximation of the average reward of the target policy using the TD error:
% The TD error for the discounted-reward formulation with centering is, at time step $t$, $\delta_t \doteq \big( R_{t+1} - \bar{R}_t + \gamma \tilde{V}^\gamma_t(S_{t+1}) - \tilde{V}^\gamma_t(S_t) \big)$.
% The idea is to update the average-reward estimate $\bar{R}_t$ using this TD error instead of the conventional error $R_{t+1} - \bar{R}_t$.
% Since this centering approach now involves values in addition to the reward, we call it \textit{value-based centering}. 
% At time step $t$, with the knowledge of $(S_t, A_t, R_{t+1}, S_{t+1})$, update the value and average-reward estimates as:
\begin{align}
    \tilde{V}^\gamma_{t+1}(S_t) &\doteq \tilde{V}^\gamma_t(S_t) + \alpha_t\, \rho_t\, \delta_t, \label{eq:update_centeredTD_offpolicy_values} \\ 
    \bar{R}_{t+1} &\doteq \bar{R}_t + \eta\, \alpha_t\, \rho_t\, \delta_t, \label{eq:update_centeredTD_offpolicy_rbar_TDerror}
\end{align}
where, $\delta_t \doteq (R_{t+1} - \bar{R}_t) + \gamma \tilde{V}_t^\gamma(S_{t+1}) - \tilde{V}_t^\gamma(S_{t})$ is the TD error and $\rho_t \doteq \pi(A_t|S_t)/b(A_t|S_t)$ is the importance-sampling ratio.
Since this centering approach involves values in addition to the reward, we call it \textit{value-based} centering. 
Unlike with simple centering, the convergence of the average-reward estimate and the value estimates is now interdependent.
We present a convergence result in the next section for the control problem.

% Figure \ref{fig:centering_results_prediction_offpolicy} shows a comparison of simple and value-based centering. 
% The goal is again to evaluate the target policy from Section \ref{sec:simple_centering}, now using the behavior policies that have different probabilities of picking the left and right actions: $[b_1(\text{left}|\cdot), b_1(\text{right}|\cdot)] = [0.7, 0.3]$, $[b_2(\text{left}|\cdot), b_2(\text{right}|\cdot)] = [0.3, 0.7]$.
% % The evaluation metric was still the root mean-squared value error (RMSVE) of the estimates and the target policy's values w.r.t.~the state distribution induced by the target policy.
% % The evaulation metric was the same as before: RMSVE.
% We tested off-policy TD-learning \eqref{eq:update_centeredTD_offpolicy_values} with both kinds of reward centering: the simple kind with an importance-sampling ratio in the update \eqref{eq:update_centeredTD_rbar_simple}, and value-based centering \eqref{eq:update_centeredTD_offpolicy_rbar_TDerror}.
% % Each parameter configuration was run for 50,000 steps and repeated 50 times.
% As baselines, we also ran standard off-policy TD without centering and with rewards centered by an oracle.
% In addition, we also tested value-based centering in the on-policy setting—using a behavior policy that is same as the target policy.
% Figure \ref{fig:centering_results_prediction_offpolicy} shows learning curves for these experiments: each column shows the plots for a particular on- or off-policy problem; the two rows correspond to two discount factors.

The first column of Figure \ref{fig:centering_results_prediction_offpolicy} shows plots for value-based centering in the on-policy problem from the previous section, where the target policy picks both actions with equal probability.
Value-based centering (red) appears as good as simple centering (green) in terms of the rate of learning and asymptotic error. 
%using the TD error to estimate the average reward works well in the on-policy setting.
The other two columns show plots for two off-policy experiments with behavior policies $[b_1(\text{left}|\cdot), b_1(\text{right}|\cdot)] = [0.7, 0.3]$, $[b_2(\text{left}|\cdot), b_2(\text{right}|\cdot)] = [0.3, 0.7]$.
The two different behavior policies are symmetric but resulted in different trends. 
% Recall that the left side of the MDP has the smaller reward. 
% The rightmost state (with the larger reward) has the highest value, so it contributes more to the RMSVE compared to the leftmost state (note the two extreme states have the same weighting and the initial estimates are zero). 
% $b_1$ results in the agent spending more time in the left side of the Markov chain, hence the reduction in RMSVE is relatively slower compared to when behaving with $b_2$.
Corresponding to $b_1$, we saw that value-based centering resulted in a lower RMSVE faster than simple centering for both values of $\gamma$, and the final error rate was roughly the same.
As expected, the simple approach estimated the average reward incorrectly and hence the learned values were relatively larger than with value-based centering (but not as large as when there was no centering). 
% This is expected; in the previous section we discussed that the centered discounted value function has infinite solutions corresponding to different values of the average-reward estimate. 
% Something more interesting happened with $b_2$. 
The results with $b_2$ were more interesting.
The RMSVE reduced rapidly at first with simple centering, then rose sharply, and then reduced again.
This is because the average-reward estimate was initialized to zero and it converged to around $0.5$ (because $b_2$ skews the agent's state distribution towards the more-rewarding right side).
When the estimate passed the true value of $0.25$, the RMSVE was quite low, however, the estimate quickly climbed to $0.5$, resulting in the peak in RMSVE. 
Eventually the value estimates settled to values corresponding to an average-reward estimate of around $0.5$.
In contrast, %value-based centering learned an average-reward estimate %of just over $0.2$, 
the average-reward estimate was much closer to the true value when using value-based centering, 
resulting in a smoother learning curve.
The effects were amplified with the larger discount factor (bottom row).
% In both the off-policy problems, value-based reward centering resulted in a higher rate of learning compared to the uncentered version, for both values of $\gamma$.

Overall, we observed that reward centering can improve the rate of learning of discounted-reward prediction algorithms such as TD-learning, especially for large discount factors.
While the simple way to center rewards is quite effective, value-based reward centering is better suited for general off-policy problems.
Next, we consider reward centering within the control setting.

%%%%%%%%%%%%%%%%%%%%%%%%%%%%%%%%%%%%%%%%%%%%%%%%%%%%%%%%%%%%%%%%

%\pdfbookmark[0]{Case Study: Q-learning with Reward Centering}{}
\section{Case Study: Q-learning with Reward Centering}
\label{sec:case_study}

In this section, we examine the effects of reward centering when used alongside the Q-learning algorithm (Watkins \& Dayan, 1992).
In particular, we first present a convergence result based on recent work by Devraj and Meyn (2021).
% We highlight important related work and analyze the convergence as well as the fixed point of tabular Q-learning with value-based reward centering. % and quantify how close the fixed point is to the true average reward and centered discounted value function.
% We will also highlight important related work.
% Beyond the theory, 
Next, using various control problems, we empirically study the effects of reward centering on tabular, linear, and non-linear variants of Q-learning. 
% We will also highlight how our approach differs from related work. \red{rewrite this}

%%%%%%%%%%%%%%%%%%%%%%%%%%%%%%

\textbf{Theory}: 
% We begin by specifying how to use Q-learning with reward centering. 
% Q-learning is one of the oldest and most prevalent control algorithms. 
The prevalence of Q-learning can be largely attributed to it being an off-policy algorithm: in the tabular case, it is guaranteed to converge to the value function of optimal policy while collecting data from an arbitrary behavior policy—even a random policy.
Given its off-policy nature, we augment Q-learning with value-based reward centering.
Since we use tabular, linear, and non-linear versions of this algorithm, we present a general form of its updates.
At each time step, given an observation, the agent converts it into a feature vector $\vx_t\in\mathbb{R}^d$, selects an action $A_t$, observes the reward signal $R_{t+1}$ and the next observation, which it converts into $\vx_{t+1}$, and so on.
In the tabular case, $\vx_t$ is a one-hot vector of the size of the state space; in the linear case, $\vx_t$ may be a tile-coding representation; in the non-linear case, $\vx_t$ is the output of the last non-linear layer of an artificial neural network. 
In each case, the agent linearly combines the feature vector with an action-specific weight vector $\vw^a\in\mathbb{R}^d, \forall a$ to obtain the action-value estimate $\hat{q}$. % corresponding to a feature vector and an action.
% At time step $t$, with the knowledge of $(\vx_t, A_t, R_{t+1}, \vx_{t+1})$, update the average-reward estimate and the per-action weights:
At time step $t$, with the knowledge of transition $(\vx_t, A_t, R_{t+1}, \vx_{t+1})$, Q-learning with value-based reward centering updates the average-reward estimate and the per-action weights:
\vspace{-1mm}
\begin{align}
    \vw_{t+1}^{A_t} &\doteq \vw_t^{A_t} + \alpha_t\,\delta_t\,\nabla_{\vw_t}\hat{q}(\vx_t, A_t),\label{eq:update_cdiscq_action_values} \\
    \bar{R}_{t+1} &\doteq \bar{R}_t + \eta\,\alpha_t\,\delta_t, \label{eq:update_cdiscq_rbar_TDerror} \\
    \text{where,}\quad \delta_t &\doteq R_{t+1} - \bar{R}_t + \gamma \max_a (\vw_t^a)^\top\vx_{t+1} - (\vw_t^{A_t})^\top\vx_t. \nonumber
\end{align}
% Since the average-reward update involves the values and not just the rewards, the convergence of the average-reward estimate and the value estimates cannot be analyzed separately. %its convergence cannot be analyzed separately from that of the value estimates.
% Fortunately, 
% We leverage some recent work to show that Q-learning with value-based reward centering converges almost surely in the tabular case. 
The full pseudocode for all algorithms is in Appendix \ref{app:pseudocode}. 
We present the informal convergence-theorem statement here; the full theorem statement, proof, and analysis are in Appendix \ref{app:theory}.
\vspace{1mm}
\begin{theorem}
    If the Markov chain induced by the stationary behavior policy is irreducible and a per-state--action step size is reduced appropriately, tabular Q-learning with value-based reward centering (\ref{eq:update_cdiscq_action_values}–\ref{eq:update_cdiscq_rbar_TDerror}) converges almost surely: $Q_t$ and $\bar{R}_t$ converge to a particular solution $(\tilde{q}^\gamma, \bar{r})$ of the following Bellman equations: 
    \vspace{-2mm}
    \begin{align}
        \tilde{q}^\gamma(s,a) = \sum_{s',r} p(s',r \mid s,a) \big( r - \bar{r} + \gamma \max_{a'} \tilde{q}^\gamma(s',a') \big).\label{eq:bellman_equation_centered_optimality_action} %\quad \text{or} \quad \bar \vv^\gamma = \vr_\pi - \bar{r} \vone + \gamma \mP_\pi \bar \vv^\gamma,    
    \end{align}
\end{theorem}
\vspace{-3mm}
The convergence proof is a consequence of important recent work by Devraj and Meyn (2021), who showed that subtracting a quantity from the rewards in Q-learning can result in a significantly better sample-complexity bound.
Depending on the quantity subtracted, there is a whole family of Q-learning variants that converge almost surely in the tabular case to $\tilde{\vQ}_\infty^\gamma = \vq^\gamma_* - k/(1-\gamma) \vone$, where $\tilde{\vQ}_\infty^\gamma$ denotes the vector of asymptotic value estimates, $\vq^\gamma_*$ denotes the discounted action-value function of the optimal policy $\pi_\gamma^*$ corresponding to the discount factor $\gamma$, %\footnote{
% Recall that in the tabular case, there are optimal policies corresponding to each discount factor in $[0,1)$, however, they may not maximize the average reward. 
% Consider the case of $\gamma=0$ for intuition. 
% The optimal policy corresponding to $\gamma=0$ will in general not maximize the total reward in an full-RL episodic problem or the average reward in a full-RL continuing problem.}
% and $k$ depends on $\kappa, \mu$ and $\vq^\gamma_*$.
and $k$ depends on $\vq^\gamma_*$ and two algorithm parameters $\mu$ and $\kappa$.
Recall that the standard (uncentered) discounted value function $\vq^\gamma_*$ has a state--action-independent offset of $r(\pi_\gamma^*)/(1-\gamma)$.
Relative Q-learning can remove $k/(1-\gamma)$ of it.
This is very promising.
Devraj and Meyn left the choice of $\mu$ and $\kappa$ as open questions. We show that Q-learning with value-based centering can be seen as an instance of their algorithm family with particular choices of $\mu$ and $\kappa$. 
We further show (in Appendix \ref{app:theory}) that these choices %of $\kappa$ and $\mu$ 
can significantly reduce the %effect of 
state-independent offset. 
The equivalence enabled us to use their theoretical machinery to show almost-sure convergence and inherit strong variance-reduction properties. 

\textbf{Experiments:} 
We present results of Q-learning with and without centering on a set of control problems with tabular, linear, and non-linear function approximation (see Appendix \ref{app:pseudocode} for the pseudocode). 
The problems are primarily from CSuite (Zhao et al., 2022). %; we shall highlight the small modifications that we made to some of the problems.
The %CSuite 
repository specifies each problem in detail; we provide high-level descriptions here.
We start the assessment in a tabular problem and then proceed to problems that require function approximation.

The Access-Control Queuing problem (Sutton \& Barto, 2018) is a continuing problem in which  the agent manages the access of incoming jobs to a set of servers. 
A job arrives at the front of the queue with one of four priorities with equal probability, and the agent has to decide at each time step whether to accept or reject the job based on the number of free servers left. 
If a job is accepted, the agent gets a positive reward proportional to the job's priority ($\{1,2,4,8\}$); if rejected, the job is removed from the queue and the agent gets zero reward. 
At each time step, occupied servers get free with a certain probability, and the agent can observe the number of servers that are currently free as well as the priority of the job at the front of the queue.

Figure \ref{fig:effect_of_gamma} shows the results of standard Q-learning (without centering) and Q-learning with value-based centering. 
For Q-learning, the curves correspond to the step-size parameters that resulted in the fastest learning over the training period (quantified by the area under the learning curve).
For Q-learning with centering, they correspond to the best step-size parameters for a fixed value of $\eta$ (shown in grey in the figure); this does not always mean the best $(\alpha, \eta)$ pair but that is okay since the results were robust to the choice of $\eta$. 
Throughout this section we followed this same practice of picking hyperparameters to plot learning curves.

The performance of Q-learning with centering did not degrade when the discount factor was close to one, unlike when there was no centering. 
For each discount factor, the performance with centering matched or exceeded that of the standard uncentered method.
To verify if centering indeed helped remove the potentially large state-independent term, we checked the magnitude of the learned values. 
\begin{wraptable}{r}{53mm}
    \vspace{-1mm}
    \caption{Magnitude of learned values}
    \label{tab:magnitude_accesscontrol}
    \vspace{-3mm}
    \begin{tabular}{ccc}
        % $\gamma$ & \textbf{DiscQ} & \textbf{CDiscQ} \\ \midrule
        \multicolumn{1}{c}{$\gamma$} & \textbf{\begin{tabular}[c]{@{}c@{}}Without\\centering\end{tabular}} & \textbf{\begin{tabular}[c]{@{}c@{}}With\\centering\end{tabular}} \\ \midrule
        0.5 & 4.78 & 0.17 \\
        0.8 & 12.95 & 0.17 \\
        0.9 & 26.57 & 0.12 \\
        0.99 & 267.91 & 0.42 \\
        0.999 & 1434.47 & 0.51 \\
    \end{tabular}
    \vspace{-3mm}
\end{wraptable}
One way is to compute the magnitude across all state-action pairs.
However, this approach typically leads to a poor approximation of the magnitude of learned values because many states (especially ones with low true values) may not occur frequently in the agent's $\epsilon$-greedy interactions with the environment and hence their estimated values may stay close to their initialization. 
Instead, we checked the values of states that actually occur in the agent's stream of experience, in particular the maximum action value (used to choose the argmax action) of the last $10\%$ states that occurred during training.   
Table \ref{tab:magnitude_accesscontrol} shows these values for the parameters corresponding to Figure \ref{fig:effect_of_gamma}'s learning curves. 
As $\gamma$ increased, the magnitude of learned values increased sharply with standard Q-learning but remained small with reward centering (as expected from the theory in Section \ref{sec:the_idea}). 

\begin{figure}[b]
    \vspace{-7mm}
    \centering
    \includegraphics[width=0.9\textwidth,trim={0 0 2mm 0},clip]{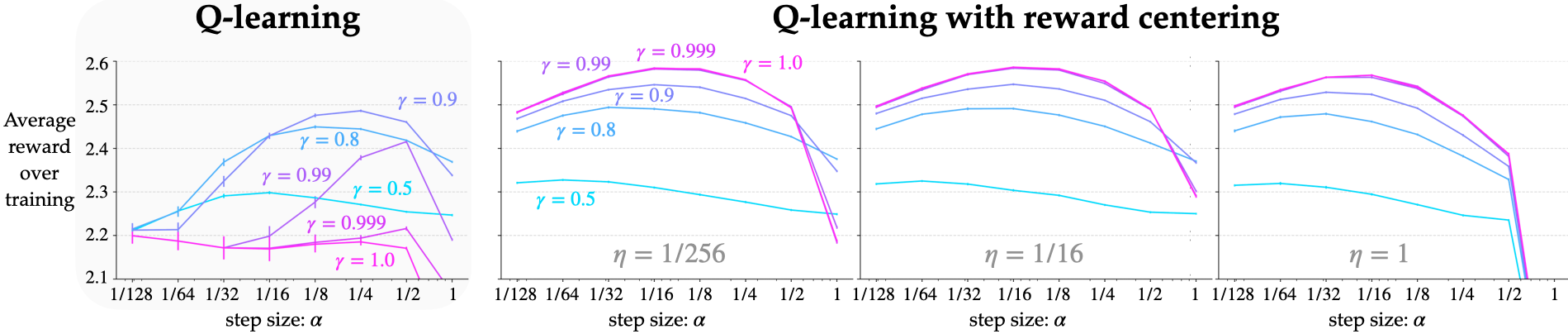}
    \caption{Parameter studies showing the sensitivity of the algorithms’ performance to their parameters on the Access-Control problem. The error bars indicate one standard error, which at times is less than the width of the lines. %\textit{Far left:} Without centering, Q-learning's rate of learning deteriorated with large discount factors for a broad range of the step-size parameter $\alpha$. \textit{Center to right:} For each discount factor, the performance with centering was better across a broad range of $\alpha$; the performance was additionally as well as robust across a large range of $\eta$.
    }
    \label{fig:results_accesscontrol_sensitivity_offset0}
    \vspace{-3mm}
\end{figure}

These trends were quite general across the range of parameter values tested. 
Figure \ref{fig:results_accesscontrol_sensitivity_offset0} shows the performance sensitivity to the methods' parameters. 
In particular, the $x$-axis denotes the step-size parameter $\alpha$ and the $y$-axis denotes the average reward obtained during the entire training period (which reflects the rate of learning).
For both methods, the different curves correspond to different discount factors. 
The three plots on the right correspond to different values of the centering step-size parameter $\eta$. 
We saw the performance of Q-learning without centering deteriorated with large discount factors for a broad range of the step-size parameter $\alpha$.
In contrast, with centering, the performance did not degrade; in fact, it improved all the way till $\gamma=1$ for a wide range of $\eta$ values. 
In addition, its performance was not sensitive to the choice of $\eta$. 

We also observed the rate of learning of the standard Q-learning algorithm is significantly affected by a constant shift in the \textit{problems'} rewards.
Note that adding a constant to all the rewards does not change the ordering of the policies according to the total-reward or the average-reward criterion in continuing problems. % where the agent-environment interactions go on ad infinitum.
% However, Q-learning is susceptible to this change. 
Figure \ref{fig:effect_of_offset} shows the behaviors of Q-learning with and without centering when applied to five problem variants with one of $\{-8, -4, 0, 4, 8\}$ added to all the rewards. %, which does not change the ordering of the policies in a continuing problem. 
To compare the resulting rate of rewards across the problems, the plots are shifted post-hoc (for instance, in the problem variant where the rewards were shifted by $8$, after training, the same number was subtracted from all the rewards that the agent obtained).
The behavior of Q-learning without centering was substantially different on all the problem variants.
Q-learning with centering, unsurprisingly, results in similar behavior.
We verified that the average-reward estimate indeed learns the average reward for every variant quickly.
These trends were also consistent across values of the step-size parameters (the parameter studies are in Appendix \ref{app:more_results}). 

\begin{figure}[t]
    \vspace{-4mm}
    \centering
    \includegraphics[width=0.72\textwidth,trim={0 0 1mm 0},clip]{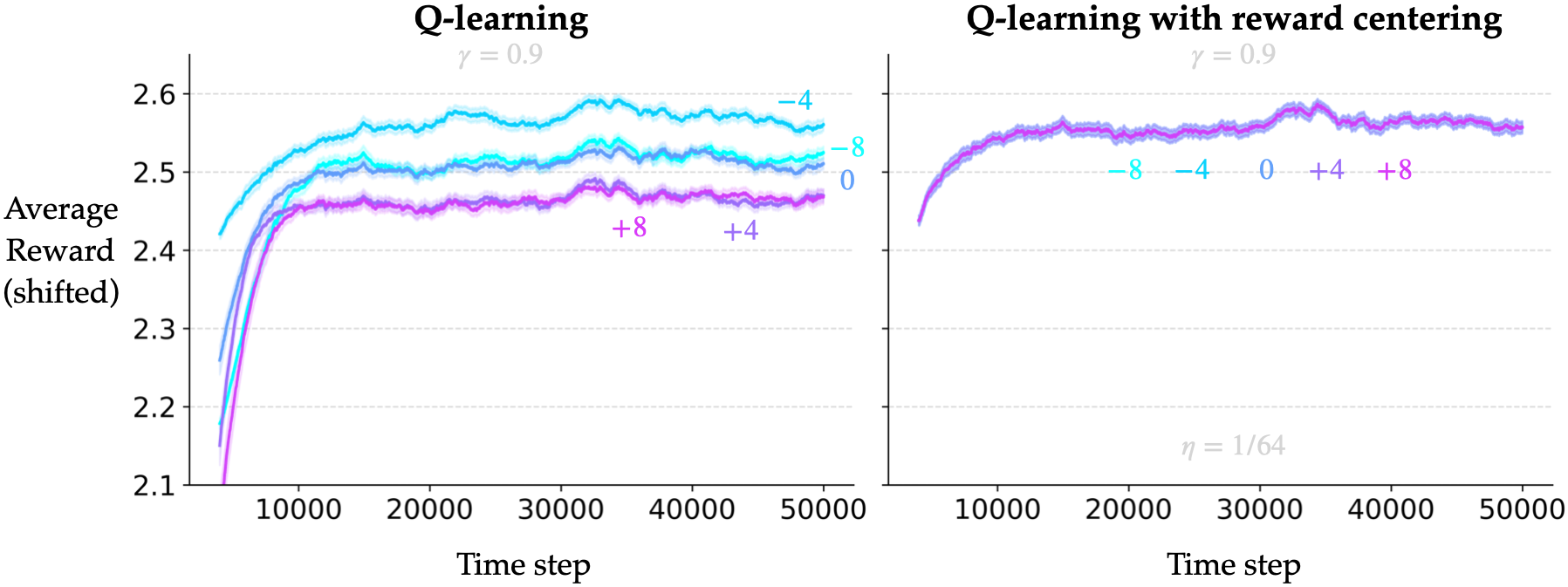}
    \vspace{-2mm}
    \caption{Learning curves on slight variants of the Access-Control Queuing problem with all the rewards shifted by a constant integer. The $y$-axis is shifted to compare learning curves for all the variants on the same scale. More details in-text.}
    \label{fig:effect_of_offset}
    \vspace{-5mm}
\end{figure}

We observed similar trends on other continuing problems with linear and one with non-linear function approximation. 
% Figure \ref{fig:results_PW_Pen_Catch_learning_curves} shows the results on three continuing problems. %: two with linear and one with non-linear function approximation. 
% now PW/Pen/catch gamma plts
\begin{figure}[!b]
    % \vspace{-2mm}
    \begin{subfigure}{.5\textwidth}
        \centering
        \includegraphics[width=0.99\textwidth]{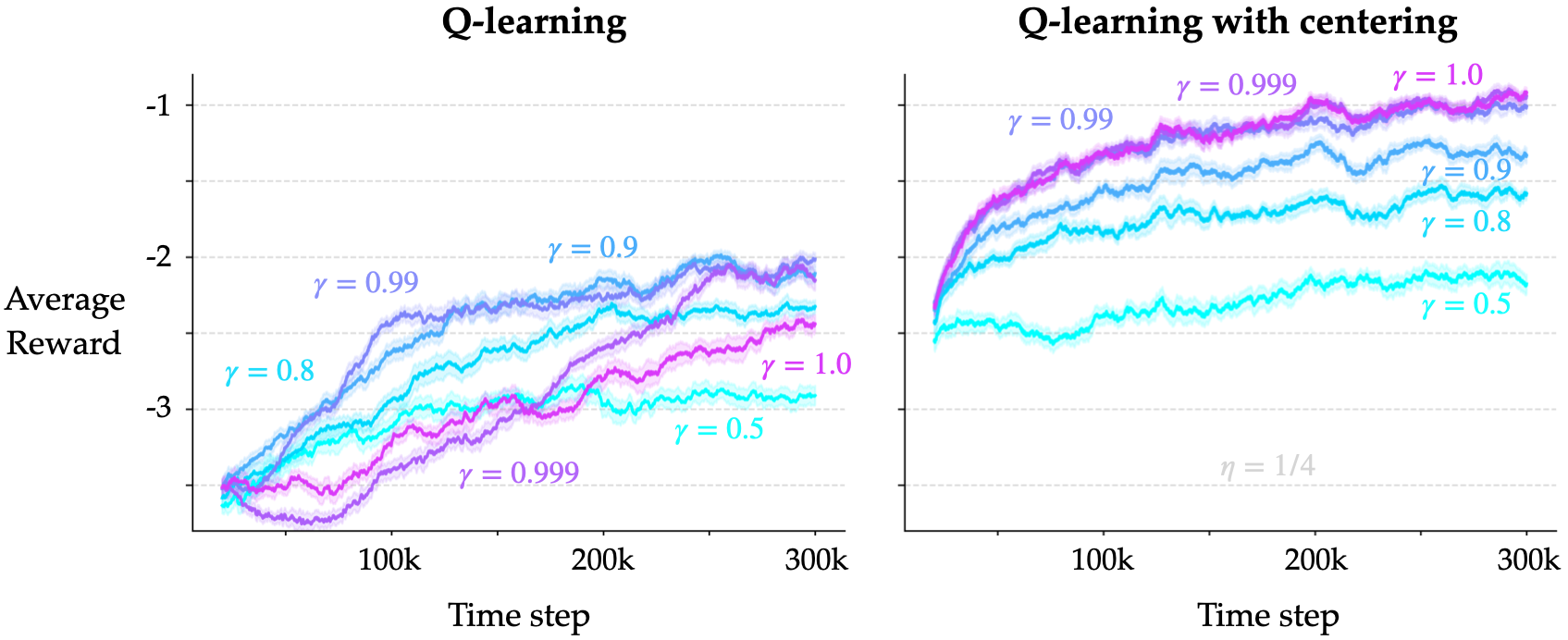}
        \vspace{-3mm}
        \caption{PuckWorld}
        \label{fig:results_puckworld_effect_of_gamma}
    \end{subfigure}%
    \begin{subfigure}{.5\textwidth}
        \centering
        \includegraphics[width=0.99\textwidth]{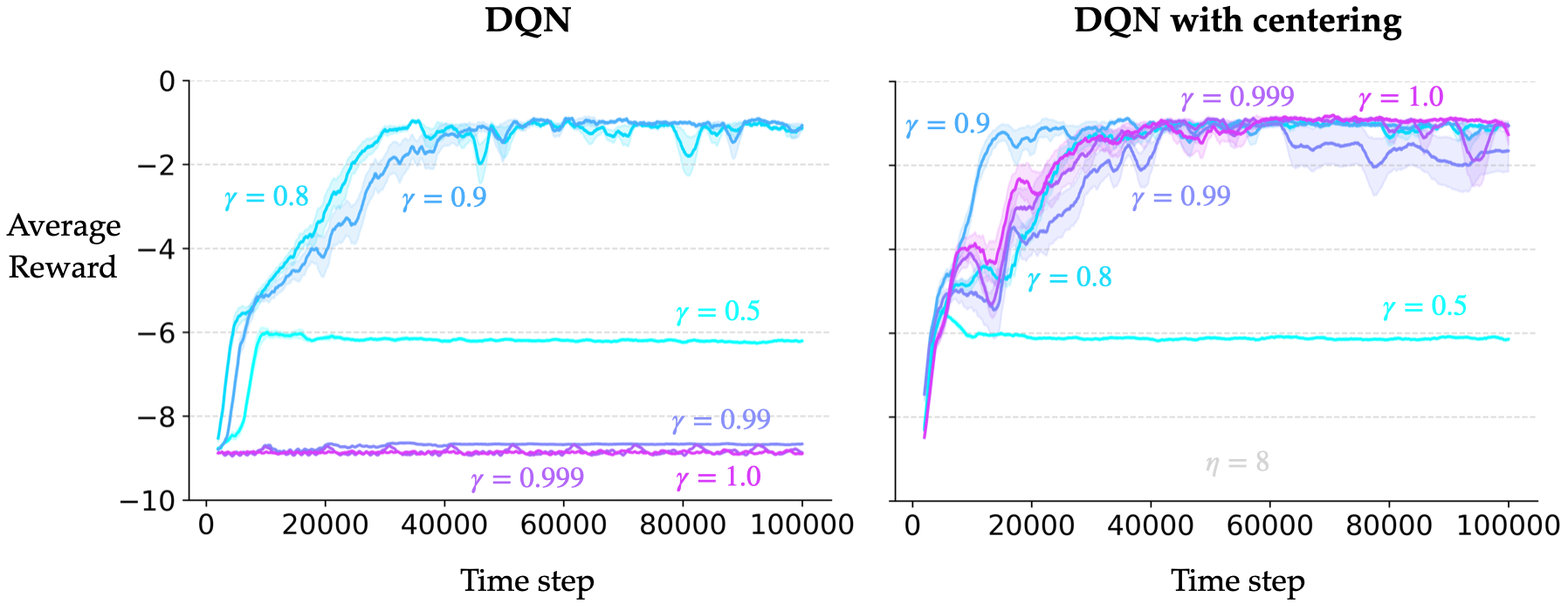}
        \vspace{-3mm}
        \caption{Pendulum}
    \end{subfigure}
    \begin{subfigure}{\textwidth}
        \centering
        \includegraphics[width=0.9\textwidth,trim={0 0 1mm 0},clip]{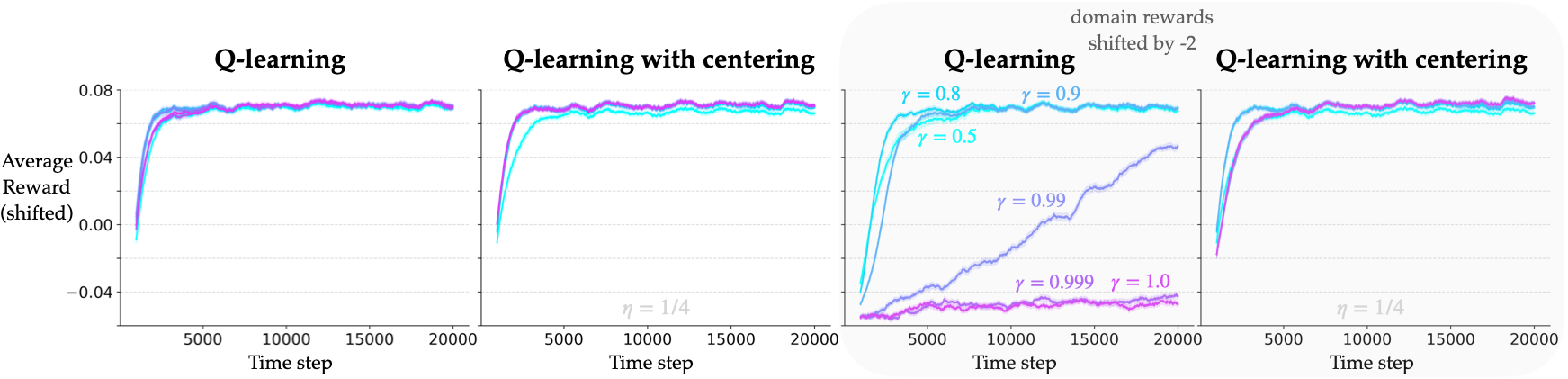}
        \vspace{-2mm}
        \caption{Catch}
    \end{subfigure}
    % \vspace{-2mm}
    \caption{Learning curves with and without centering corresponding to different values of $\gamma$ on different problems. %The two plots on the bottom right correspond to a variant of the Catch problem with all the rewards shifted by $-2$.
    In the bottom row, the two plots on the right correspond to a variant of the Catch problem where all the rewards shifted by $-2$.%; the plots on the left correspond to the original rewards.
    }
    \label{fig:results_PW_Pen_Catch_learning_curves}
\end{figure}
In PuckWorld, the agent has to take a puck-like object to randomly changing goal positions in a square rink. 
At each time step, the agent observes six real numbers—the puck's position and velocity and the goal position in $x$ and $y$ directions—and gets a reward proportional to the negative distance to the goal.
% In PuckWorld, the agent controls a puck-like object in a square rink where goal positions occur randomly.
% The agent can push the puck in any of the cardinal directions, where repeated actions in a direction give the puck some velocity that is upper-bounded due to friction. 
% The agent observes six real numbers at each time step—the puck's position and velocity and the goal position in $x$ and $y$ directions—and gets a reward proportional to the negative distance to the goal.
% % The best policy goes to the goal position as soon as possible which moves to a new random location every 300 time steps.
In Pendulum, the agent has to control the torque at the base of a one-link pendulum to take and maintain it in an upright position. 
At each time step, the agent observes three real numbers—the sine and cosine of the pendulum's angle w.r.t.~the direction of gravity, and the pendulum's angular velocity—and gets a reward proportional to the negative angular distance of the pendulum from the upright position. 
% The agent controls the torque at the base of a one-link pendulum and gets a reward at each time step proportional to the negative angular distance of the pendulum from the upright position. 
% The pendulum starts at rest pointing down. The agent can only apply a discrete amount of torque of $\{-1, 0, 1\}$ unit at each time step
% after observing three real numbers: the sine and cosine of the pendulum's angle w.r.t.~pointing downwards, and the pendulum's angular velocity. 
% There are no resets or timeouts; the agent must learn to keep the pendulum in the upright position, which falls repeatedly because the upright position is an unstable equilibrium and any exploratory actions can upset the pendulum. 
In Catch, the agent moves a crate in the bottom row of a 2D pixel grid to catch falling fruits. 
For this problem, there are two kinds of observation vectors available to an agent: a 3D real vector containing the $x$ coordinate of the crate and the $(x,y)$ coordinates of the lowermost fruit; a 50D binary vector which is the flattened version of the entire %$10\times 5$ 
pixel grid.
The agent gets a $+1$ reward on successfully catching a fruit, $-1$ on dropping one, and $0$ otherwise. 
All the problems are continuing; there are no resets. %, terminations, or timeouts.
% In Catch, the agent controls a crate at the bottom row of a 2D pixel grid to catch falling fruits. 
% The agent can move the crate one pixel right, left, or stay put, and gets a +1 reward on successfully catching a fruit, -1 on dropping one, and 0 otherwise. 
% At each time step, a new fruit is spawned with 10\% probability in the top row, in a random column.
% More than one fruit may be falling at any point of time, and each fruit falls one pixel in one time step.
% There are two kinds of observation vectors available to an agent: a 3-dimensional real vector containing the $x$ coordinate of the crate and the $(x,y)$ coordinates of the lowermost fruit; a 50-dimensional binary vector which is the flattened version of the entire $10\times 5$ pixel grid.

We used linear function approximation with tile-coded features for PuckWorld and the variant of Catch in which the agent observes the 3D real-valued features.
For Pendulum and the variant of Catch with the 50D binary features, we used non-linear function approximation using artificial neural networks (Mnih et al.'s (2015) DQN). 
The complete experimental details are in Appendix \ref{app:more_results}.

The trends were similar to those observed with Access-Control Queuing. 
In PuckWorld and Pendulum (top row of Figure \ref{fig:results_PW_Pen_Catch_learning_curves}), without centering, performance first improved as the discount factor $\gamma$ increased and then degraded.
However, with centering, the performance did not degrade for large values of $\gamma$.
In Catch with linear function approximation (bottom row of Figure \ref{fig:results_PW_Pen_Catch_learning_curves}), the leftmost plot shows that the performance without centering was good even for large discount factors.
However, it varied significantly when the problem rewards were shifted up or down by a constant; the third plot from the left demonstrates this for a shift of $-2$. 
On the other hand, with centering, the performance was good for all discount factors and unaffected by any shifts in the rewards.
% This observation 

These trends are further supplemented by the two plots on the left of Figure \ref{fig:results_catch_sensitivity_gamma0.8}, which shows the sensitivity of the algorithms to variants of the Catch problem with rewards shifted by a constant. 
On the $x$-axis is the effective step size for the linear function approximators and on the $y$-axis is the reward rate averaged over the entire training period.
As before, the $y$-axis is adjusted to compare the performance on all the problem variants at the same scale.
We observed that the performance without reward centering was problem-dependent, whereas with centering, the rate of learning was roughly the same regardless of the problem variant.
The two plots on the right of Figure \ref{fig:results_catch_sensitivity_gamma0.8} show that the trends were similar with non-linear function approximation.
% These trends suggest that uncentered Q-learning may work well if the rewards 

\begin{figure}[h]
    \vspace{-2mm}
    \centering
    \includegraphics[width=0.9\textwidth]{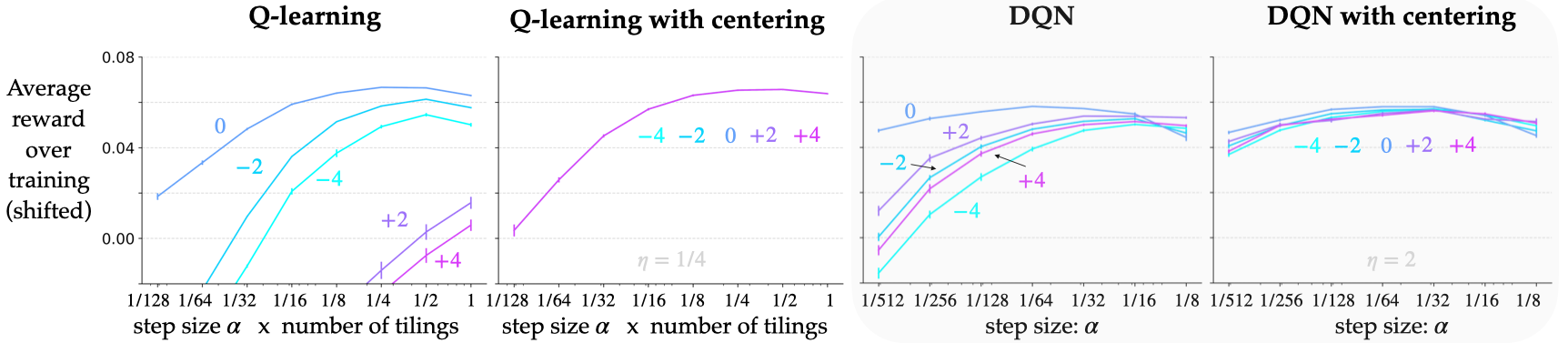}
    \caption{Parameter studies showing the sensitivity of the algorithms to their step-size parameter and to variants of the Catch problem, using both linear and non-linear function approximation. %\textit{Left:} Q-learning's rate of learning depended strongly on the problem variant and the step size, whereas that of Centered Q-learning did not depend on the problem variant. \textit{Right:} The rate of learning of Centered DQN was roughly independent of the problem variant as well as the step-size parameter while that of DQN depended on both.
    }
    \label{fig:results_catch_sensitivity_gamma0.8}
    \vspace{-1mm}
\end{figure} 

Through these experiments, we observed that reward centering can improve the performance of tabular, linear, and non-linear variants of the Q-learning algorithm on various problems.
The improvement in the rate of learning is larger for discount factors close to 1.
Furthermore, there is an improvement in the robustness of the algorithms to shifts in the problems' rewards.
The parameter studies in this section indicate that the benefits of reward centering are quite robust to the choice of its parameter $\eta$.
Appendix \ref{app:more_results} contains additional learning curves and parameter studies that further reinforce the trends observed in this section.

% Access-Control Queuing, PuckWorld, Pendulum, and Catch, , with tabular, linear, and , we have seen that reward centering can improve the performance of Q-learning for all discount factors, especially as $\gamma\to 1$. 

%%%%%%%%%%%%%%%%%%%%%%%%%%%%%%%%%%%%%%%%%%%%%%%%%%%%%%%%%%%%%%%%

%\pdfbookmark[0]{Discussion, Limitations, and Future Work}{}
\section{Discussion, Limitations, and Future Work}
\label{sec:discussion}

Reward centering can improve the data efficiency and robustness of almost any algorithm for continuing reinforcement learning.
Here we have shown improvements for algorithms that learn state-value functions and action-value functions, and for algorithms that are tabular or use linear or non-linear function approximation.\footnote{
As a further example, in preliminary experiments in Appendix \ref{app:more_results} we found reward centering to increase the data efficiency and robustness of Schulman et al.'s (2017) PPO algorithm on continuing versions of several of Todorov et al.'s (2012) Mujoco problems.}
We expect reward centering would also improve the performance of algorithms that learn no value function at all, such as REINFORCE (Williams, 1992) when applied to continuing problems with eligibility traces, but this has yet to be shown.
Many algorithms that were designed for the average-reward criterion already include a form of reward centering, either simple centering (e.g., Tsitsiklis \& Van Roy, 1999)
or value-based centering (Wan et al., 2021); it is experience with these earlier un-discounted algorithms that led us to explore the utility of reward centering for discounted algorithms.
% Other average-reward algorithms do not center their rewards (e.g., Wheeler \& Narendra, 1986), but probably would perform better if they did, though again this has yet to be shown.
Also expected, but yet to be shown, are the benefits of reward centering with other reinforcement-learning algorithms, including value-based algorithms such as Sarsa (Rummery \& Niranjan, 1994), and various offline, actor--critic, and model-based algorithms.
%actor--critic, policy-gradient, and model-based reinforcement learning algorithms.

Reward centering is not directly applicable to episodic problems. 
% In these problems the objective is to maximize the sum of rewards only up until the end of an episode, in which case the Laurent-series decomposition \eqref{eq:laurent_series_expansion_state_values} no longer holds (there is no state-independent term). 
In these problems the objective is to maximize the sum of rewards only up until the end of an episode; the notion of long-term average reward is undefined and the Laurent-series decomposition (with a state-independent term) no longer holds.
Moreover, if reward centering were naively applied to an episodic problem, then it may alter the problem rather than facilitate finding the solution.
This is because—unlike in continuing problems—subtracting a constant from all the rewards may change an episodic problem.
For example, consider a gridworld where the reward is $-1$ on every step until episode termination upon reaching a goal state.
An optimal policy---one that maximizes the total reward per episode---is one that reaches the goal state as soon as possible.
However, if the rewards were centered, then the modified rewards would all be zero, and all policies would be equally optimal.
The \textit{problem} would be fundamentally altered by the \textit{algorithm} that centers (or in general shifts) rewards!
The closest thing to reward centering in episodic problems may be the return baseline in policy-gradient methods (e.g., see Sutton \& Barto, 2018, Section 13.4), but that may vary from state to state, so the analogy is not really that close.\footnote{In continuing problems, reward centering has orthogonal benefits to the baseline or the advantage function in policy-gradient methods. We discuss these and further connections to the literature in Appendix \ref{app:related_work}.}

Reward centering may seem similar to, but is different from, a value-function unit with a bias weight (a weight for an input that is always 1).
First, the bias weight converges asymptotically to a value that depends on all the other inputs to the value-function unit and that is not in general $r(\pi)/(1-\gamma)$.
Second, because the learning of the bias weight interacts with learning all the other weights, we will not obtain the same data efficiency advantages as with reward centering. 
Reward centering is more akin to the specially learned bias weight in Sutton's (1988b) NADALINE linear unit.
We note that reward centering is also different from but similar in style to methods for adapting the \textit{scale} of the rewards (e.g., van Hasselt et al., 2016; Pohlen et al., 2018; Schaul et al., 2021), and the two kinds of methods can potentially be used together.

The effects of reward centering on the variance of value estimates is complex. 
On one hand, reward centering can increase variance because the average-reward estimate changes over time.
The simple centering method is particularly susceptible to this in the off-policy setting (e.g., see Figure \ref{fig:centering_results_prediction_offpolicy}).
% \red{if problem rewards are already centered}
On the other hand, value-based centering in particular can reduce the variance caused by state-dependent reward changes (cf.\ Sutton \& Barto, 2018, Exercise 10.8).
In all cases, optimization techniques could be used to efficiently adapt the step-size parameter of the average-reward estimate (Degris et al., 2024).

Perhaps the most exciting direction for extending reward centering is into new reinforcement learning algorithms that adapt their discount-rate parameter over time. 
Without reward centering, this would incur huge costs in learning time as the discounted values change by large amounts even for small changes in $\gamma\approx 1$. 
Most of these changes are due to the state-independent term $r(\pi)/(1-\gamma)$ which, with reward centering, can be adapted instantly to the new value of $\gamma$ using the existing estimate of $r(\pi)$.
Concretely, consider the agent has estimated the average reward $\bar{R}$ and the centered discounted value function $\tilde{v}^{\gamma}$ to some level of accuracy. 
With just this information, the agent can form an estimate of the standard discounted value function corresponding to another discount factor $\gamma'$ via $\bar{R}/(1-\gamma') + \tilde{v}^{\gamma}$.
This is an estimate, of course, but it can be improved quickly with a few samples of experience%—potentially with old experience from a buffer or a parameterized model
.
In contrast, with standard methods, it would take comparatively longer to raise the estimates to the new mean value and adapt the relative values.
Hence, with reward centering, we can imagine efficient methods that adapt their discount factors over time: a low discount rate to learn quickly amidst a lot of uncertainty (like in the beginning of training, or when something about the world changes); a higher discount rate when the world is more predictable to estimate the policy that maximizes the total reward obtained by the agent.

% \blue{say what happens when rewards are already centered}

% Reward centering is a simple idea with potentially large benefits and hence should be used alongside all discounted-reward algorithms. %\blue{modify}
%Reward centering is a simple idea with potentially large benefits. 
%When combined with appropriate step-size adaptation and reward-scaling techniques, we think it will be a key enabler for agents to learn quickly and continually over their lifetimes.

% \subsubsection*{Broader Impact Statement}
% \label{sec:broaderImpact}
% In this optional section, RLC encourages authors to discuss possible repercussions of their work, notably any potential negative impact that a user of this research should be aware of. 

\section*{Acknowledgments}
\label{sec:ack}

The authors gratefully acknowledge funding from NSERC, DeepMind, and the pan-Canadian AI program administered by CIFAR.
We thank Huizhen Yu, Arsalan Sharifnassab, Khurram Javed, and other members of the RLAI lab for discussions that helped improve the quality and clarity of the paper.
We are also grateful for the computing resources provided by the Digital Research Alliance of Canada.
% Use unnumbered third level headings for the acknowledgments. All acknowledgments, including those to funding agencies, go at the end of the paper. Only add this information once your submission is accepted and deanonymized. 

%%%%%%%%%%%%%%%%%%%%%%%%%%%%%%%%%%%%%%%%%%%%%%%%%%%%%%%%%%%%%%%%
%% NOTE: THIS MARKS THE END OF THE "MAIN TEXT"
%%%%%%%%%%%%%%%%%%%%%%%%%%%%%%%%%%%%%%%%%%%%%%%%%%%%%%%%%%%%%%%%

%%%%%%%%%%%%%%%%%%%%%%%%%%%%%%%%%%%%%%%%%%%%%%%%%%%%%%%%%%%%%%%%
%% Bibliography
%%%%%%%%%%%%%%%%%%%%%%%%%%%%%%%%%%%%%%%%%%%%%%%%%%%%%%%%%%%%%%%%
%\pdfbookmark[0]{References}{}
\section*{References}
% \bibliography{main}
% \bibliographystyle{rlc}
% \vspace{-2mm}
% {\small\input{references}}
\parskip=5pt
\parindent=0pt
\def\hangin{\hangindent=0.2in}

\hangin
Barto, A.~G., Sutton, R.~S., \& Anderson, C.~W. (1983).
\newblock Neuronlike Elements That Can Solve Difficult Learning Control
  Problems.
\newblock {\em IEEE Transactions on Systems, Man, and Cybernetics}.

% % \hangin
% % Bellman, R. E. (1957). \textit{Dynamic Programming}. Princeton University Press.

% % \hangin
% % Bertsekas, D. P., Tsitsiklis, J. N. (1996). \emph{Neuro-dynamic Programming}. Athena Scientific.

% % \hangin
% % Bhatnagar, S., Sutton, R. S., Ghavamzadeh, M., Lee, M. (2009). Natural actor--critic algorithms. \textit{Automatica, 45}(11):2471--2482.

\hangin
Blackwell, D. (1962). Discrete Dynamic Programming. \textit{The Annals of Mathematical Statistics.}

% \hangin
% Borkar, V. S. (1998). Asynchronous Stochastic Approximations. \emph{SIAM Journal on Control and Optimization}. % 36 (3):840--851.

% \hangin
% Borkar, V. S. (2009). \emph{Stochastic Approximation: A Dynamical Systems Viewpoint}. Springer.

\hangin
Borkar, V., \& Meyn, S. (2000). The ODE Method for Convergence of Stochastic Approximation and Reinforcement Learning. \emph{SIAM Journal on Control and Optimization}. %, 38}(2):447--469.

% \hangin
% Borkar, V. S., \& Soumyanatha, K. (1997). An Analog Scheme for Fixed Point Computation. I. Theory. \emph{IEEE Transactions on Circuits and Systems I: Fundamental Theory and Applications}. %, 44(4), 351-355.

% % \hangin
% % Brafman, R. I., Tennenholtz, M. (2002). R-{MAX} — a general polynomial time algorithm for near-optimal reinforcement learning. \emph{Journal of Machine Learning Research, 3}(10):213--231.

% \hangin
% Brunskill, E., \& Li, L. (2014). PAC-inspired Option Discovery in Lifelong Reinforcement Learning. \emph{International Conference on Machine Learning}.% (pp. 316-324).

% % \hangin
% % Chen, Y., Li, L., Wang, M. (2018). Scalable bilinear pi learning using state and action features. In \emph{International Conference on Machine Learning} (pp. 834-843). PMLR.

% \hangin
% Das, T. K., Gosavi, A., Mahadevan, S., \& Marchalleck, N. (1999). Solving Semi-Markov Decision Problems Using Average Reward Reinforcement Learning. \emph{Management Science}.
% %, 45}(4):560--574.

% \hangin
% Devraj, A.~\& Meyn, S. (2020). Q-learning with uniformly bounded variance: Large discounting is not a barrier to fast learning. \emph{ArXiv:2002.10301}.

\hangin
Degris, T., Javed, K., Sharifnassab, A., Liu, Y., \& Sutton, R.~S. (2024). Step-size Optimization for Continual Learning. \textit{ArXiv:2401.17401}.

\hangin
Devraj, A., \& Meyn, S. (2021). Q-learning with Uniformly Bounded Variance. \textit{IEEE Transactions on Automatic Control}. Also \emph{ArXiv:2002.10301}.

\hangin
Engstrom, L., Ilyas, A., Santurkar, S., Tsipras, D., Janoos, F., Rudolph, L., \& Madry, A. (2019). Implementation Matters in Deep RL: A Case Study on PPO and TRPO. \textit{International Conference on Learning Representations.}

\hangin
Even-Dar, E., Mansour, Y., \& Bartlett, P. (2003). Learning Rates for Q-learning. \emph{Journal of Machine Learning Research}.

\hangin
Kingma, D.~P., \& Ba, J. (2014). Adam: A Method for Stochastic Optimization. \textit{ArXiv:1412.6980}.

\hangin
Mnih, V., Kavukcuoglu, K., Silver, D., Rusu, A.~A., Veness, J., Bellemare, M.~G., Graves, A., Riedmiller, M., Fidjeland, A.~K., Ostrovski, G., Petersen, S., Beattie, C., Sadik, A., Antonoglou, I., King, H., Kumaran, D., Wierstra, D., Legg, S., \& Hassabis, D. (2015). Human-Level Control Through Deep Reinforcement Learning. \textit{Nature}.

% % \hangin
% % Mousavi, A., Li, L., Liu, Q., Zhou, D. (2020). Black-box off-policy estimation for infinite-horizon reinforcement learning. \emph{ArXiv:2003.11126.}

\hangin
Naik, A., Shariff, R., Yasui, N., Yao, H., \& Sutton, R.~S. (2019). Discounted Reinforcement Learning Is Not an Optimization Problem. \emph{Optimization Foundations for Reinforcement Learning Workshop at the Conference on Neural Information Processing Systems}. Also \textit{ArXiv:1910.02140}.

\hangin
Ng, A., Harada, D., \& Russell, S. (1999). Policy Invariance Under Reward Transformations: Theory and Application to Reward Shaping. \textit{International Conference on Machine Learning}.

% % \hangin 
% % Nachum, O., Chow, Y., Dai, B., Li, L. (2019). Dualdice: Behavior-agnostic estimation of discounted stationary distribution corrections. \emph{arXiv:1906.04733}.

\hangin
Paszke, A., Gross, S., Massa, F., Lerer, A., Bradbury, J., Chanan, G., Killeen, T., Lin, Z., Gimelshein, N., Antiga, L., Desmaison, A., Köpf, A., Yang, E.~Z., DeVito, Z., Raison, M., Tejani, A., Chilamkurthy, S., Steiner, B., Fang, L., Bai, J., \& Chintala, S. (2019). PyTorch: An Imperative Style, High-Performance Deep Learning Library. \textit{Advances in Neural Information Processing Systems}.

\hangin
Pohlen, T., Piot, B., Hester, T., Azar, M.~G., Horgan, D., Budden, D., Barth-Maron, G., van Hasselt, H., Quan, J., Večerík, M., Hessel, M., Munos, R., \& Pietquin, O. (2018). Observe and Look Further: Achieving Consistent Performance on Atari. \textit{ArXiv:1805.11593}.

\hangin
Puterman, M. L. (1994). \emph{Markov Decision Processes: Discrete Stochastic Dynamic Programming.} John Wiley \& Sons.

\hangin
Qu, G., \& Wierman, A. (2020). Finite-Time Analysis of Asynchronous Stochastic Approximation and Q-learning. \emph{Conference on Learning Theory.}

\hangin
Robbins, H., \& Monro, S. (1951). A Stochastic Approximation Method. \emph{The Annals of Mathematical Statistics}.

\hangin
Rummery, G.~A., \& Niranjan, M. (1994). \textit{On-line Q-learning Using Connectionist Systems.} Technical Report, Engineering Department, Cambridge University.

% % \hangin
% % Ren, Z., Krogh, B. H. (2001). Adaptive control of Markov chains with average cost. \emph{IEEE Transactions on Automatic Control, 46}(4):613--617.

\hangin
Schaul, T., Ostrovski, G., Kemaev, I.,  \& Borsa, D. (2021). Return-Based Scaling: Yet Another Normalisation Trick for Deep RL. \textit{ArXiv:2105.05347}.

\hangin 
Schneckenreither, M. (2020). Average Reward Adjusted Discounted Reinforcement Learning: Near-Blackwell-Optimal Policies for Real-World Applications. \emph{ArXiv:2004.00857}.

\hangin
Schulman, J., Moritz, P., Levine, S., Jordan, M., \& Abbeel, P. (2016). High-Dimensional Continuous Control Using Generalized Advantage Estimation. \textit{International Conference on Learning Representations}.

\hangin
Schulman, J., Wolski, F., Dhariwal, P., Radford, A., \& Klimov, O. (2017). Proximal Policy Optimization Algorithms. \textit{ArXiv:1707.06347}.

\hangin
Singh, S., Jaakkola, T., \& Jordan, M. (1994) Learning Without State-Estimation in Partially Observable Markovian Decision Processes. 
\textit{Machine Learning Proceedings}.

% \hangin
% Singh, S., Barto, A. G., \& Chentanez, N. (2004). Intrinsically Motivated Reinforcement Learning. \emph{Advances in Neural Information Processing Systems.}% (pp. 1281-1288).

% % \hangin 
% % Singh, S. S., Tadić, V. B., Doucet, A. (2007). A policy gradient method for semi-Markov decision processes with application to call admission control. \emph{European Journal of Operational Research}, 178(3), 808-818.

% \hangin
% Sorg, J., \& Singh, S. (2010). Linear Options. \emph{International Conference on Autonomous Agents and Multiagent Systems}.%: volume 1 (pp. 31-38).

\hangin
Sun, H., Han, L., Yang, R., Ma, X., Guo, J., \& Zhou, B. (2022). Exploit Reward Shifting in Value-Based Deep-RL: Optimistic Curiosity-Based Exploration and Conservative Exploitation via Linear Reward Shaping. \textit{Advances in Neural Information Processing Systems}.

\hangin
Sutton, R.~S. (1988a).
\newblock Learning to Predict by the Methods of Temporal Difference.
\newblock {\em Machine Learning} %, 3}(1):9--44 
(important erratum p.~377).

\hangin
Sutton, R.~S. (1988b). \textit{NADALINE: A Normalized Adaptive Linear Element That Learns Efficiently}. GTE Laboratories Technical Report. %TR88-509.4), GTE Laboratories Incorporated.

\hangin
Sutton, R.~S., \& Barto, A.~G. (1998, 2018). \emph{Reinforcement Learning: An Introduction.} First and second editions. MIT Press.

% \hangin
% Sutton, R. S., Modayil, J., Delp, M., Degris, T., Pilarski, P. M., White, A., \& Precup, D. (2011). Horde: A scalable real-time architecture for learning knowledge from unsupervised sensorimotor interaction. \emph{International Conference on Autonomous Agents and Multiagent Systems}.

% % \hangin
% % Tang, Z., Feng, Y., Li, L., Zhou, D., Liu, Q. (2019). Doubly robust bias reduction in infinite horizon off-policy estimation. \emph{ArXiv:1910.07186}.

% % \hangin
% % Thomas, P., Theocharous, G., & Ghavamzadeh, M. (2015). High-confidence off-policy evaluation. In \emph{Proceedings of the AAAI Conference on Artificial Intelligence} (Vol. 29, No. 1).

\hangin
Todorov, E., Erez, T., \& Tassa, Y. (2012). MuJoCo: A Physics Engine for Model-Based Control. \emph{International Conference on Robotics and Automation}.

% \hangin
% Tsitsiklis, J. N. (1994). Asynchronous stochastic approximation and Q-learning. \emph{Machine Learning}. %, 16(3), 185-202.

% \hangin
% Tsitsiklis, J. N., \& Van Roy, B. (1997). An analysis of temporal-difference learning with function approximation. \textit{IEEE Transactions on Automatic Control, 42}(5):674--690.

\hangin
Tsitsiklis, J.~N., \& Van Roy, B. (1999). Average Cost Temporal-Difference Learning. \emph{Automatica}. %, 35}(11):1799--1808.

\hangin
Tsitsiklis, J.~N., \& Van Roy, B. (2002). On Average Versus Discounted Reward Temporal-Difference Learning. \emph{Machine Learning}. %, 49}(2), 179--191.

% % \hangin
% % van Seijen, H., Mahmood, A. R., Pilarski, P. M., Machado, M. C., Sutton, R. S. (2016). True online temporal-difference learning. \textit{Journal of Machine Learning Research, 17}(145):1--40.

\hangin 
Van Hasselt, H., Guez, A., Hessel, M., Mnih, V., \& Silver, D. (2016). Learning Values Across Many Orders of Magnitude. \textit{Advances in Neural Information Processing Systems}.

\hangin
Wainwright, M. J. (2019). Stochastic Approximation With Cone-Contractive Operators: Sharp $\ell_\infty $-Bounds for $ Q $-learning. \textit{ArXiv:1905.06265}.

% \hangin
% Wheeler, R., Narendra, K. (1986). Decentralized Learning in Finite Markov Chains. \emph{IEEE Transactions on Automatic Control, 31}(6):519--526.

% % \hangin
% % White, A. (2015) \textit{Developing a predictive approach to knowledge}. Ph.D. dissertation, University of Alberta. 

% \hangin
% van Dijk, S. G. \& Polani, D. (2011). Grounding Subgoals in Information Transitions. \emph{IEEE Symposium on Adaptive Dynamic Programming and Reinforcement Learning}.% (ADPRL)} (pp. 105-111). IEEE.

% % \hangin 
% % Vien, N. A., Chung, T. (2007). Natural gradient policy for average cost SMDP problem. In \emph{19th IEEE International Conference on Tools with Artificial Intelligence} (ICTAI 2007) (Vol. 1, pp. 11-18). IEEE.

% \hangin
% Veeriah V., Zahavy T., Hessel M., Xu Z., Oh J., Kemaev I., van Hasselt H., Silver D., \& Singh S. (2021). Discovery of Options via Meta-Learned Subgoals. \emph{ArXiv:2102.06741}.

% \hangin
% Vien, N. A., \& Chung, T. (2008). Policy Gradient Semi-Markov Decision Process. \emph{IEEE International Conference on Tools with Artificial Intelligence}.% (Vol. 2, pp. 11-18). IEEE.

\hangin
Wan, Y., Naik, A., \& Sutton, R.~S. (2021). Learning and Planning in Average-Reward Markov Decision Processes. \textit{International Conference on Machine Learning}.

\hangin
Watkins, C.~J., \& Dayan, P. (1992). Q-learning. \textit{Machine Learning}.

% % \hangin
% % Wan, Y., Naik, A., Sutton, R. S. (2020). Learning and Planning in Average-Reward Markov Decision Processes. \emph{arXiv:2006.16318}.

% % \hangin
% % Wei, C. Y., Jahromi, M. J., Luo, H., Sharma, H., Jain, R. (2020). Model-free reinforcement learning in infinite-horizon average-reward markov decision processes. In \emph{International Conference on Machine Learning} (pp. 10170-10180). PMLR.

% % \hangin
% % White, D. J. (1963). Dynamic programming, Markov chains, and the method of successive approximations. \emph{Journal of Mathematical Analysis and Applications, 6}(3):373--376.

\hangin
Williams, R.~J. (1992). Simple Statistical Gradient-Following Algorithms for Connectionist Reinforcement Learning. \textit{Machine Learning}.

\hangin
Zhao, R., Abbas, Z., Szepesvári, D., Naik, A., Holland, Z., Tanner, B., \& White, A. (2022). CSuite: Continuing Environments for Reinforcement Learning, \textit{Github:} \href{https://github.com/google-deepmind/csuite}{google-deepmind/csuite}

%%%%%%%%%%%%%%%%%%%%%%%%%%%%%%%%%%%%%%%%%%%%%%%%%%%%%%%%%%%%%%%%
%% Appendices
%%%%%%%%%%%%%%%%%%%%%%%%%%%%%%%%%%%%%%%%%%%%%%%%%%%%%%%%%%%%%%%%
\appendix

%%%%%%%%%%%%%%%%%%%%%%%%%%%%%%%%%%%%%%%%%%%%%%%%

\section{Pseudocode}
\label{app:pseudocode}

In this section we present the pseudocode for value-based reward-centering added to the tabular, linear, and non-linear variants of Q-learning.

\vspace{-2mm}
\begin{algorithm}[h]
% \DontPrintSemicolon
\SetAlgoLined
\KwIn{The behavior policy $b$ (e.g., $\epsilon$-greedy)}
\SetKwInput{AP}{Algorithm parameters}
\AP{discount factor $\gamma$, step-size parameters $\alpha, \eta$}
Initialize $Q(s,a)\ \forall s,a; \bar{R}$ arbitrarily (e.g., to zero) \\
Obtain initial $S$ \\
 \For{all time steps}
 {
    Take action $A$ according to $b$, observe $R, S'$ \\
    $\delta = R - \bar{R} + \gamma \max_{a}Q(S',a) - Q(S,A)$ \\
    $Q(S,A) = Q(S,A) + \alpha\, \delta$ \\
    $\bar{R} = \bar{R} + \eta\, \alpha\, \delta$ \\
    $S = S'$ \\
 }
 \caption{Tabular Q-learning with value-based reward centering}
 \label{algo:CDiscQ_tabular}
\end{algorithm}

\vspace{-2mm}
\begin{algorithm}[h]
% \DontPrintSemicolon
\SetAlgoLined
\KwIn{The behavior policy $b$ (e.g., $\epsilon$-greedy)}
\SetKwInput{AP}{Algorithm parameters}
\AP{discount factor $\gamma$, step-size parameters $\alpha, \eta$}
Initialize $\vw_a \in \mathbb{R}^d \,\forall a, \bar{R}$ arbitrarily (e.g., to zero) \\
Obtain initial observation $\vx$ \\
 \For{all time steps}
 {
    Take action $A$ according to $b$, observe $R, \vx'$ \\
    $\delta = R - \bar{R} + \gamma \max_{a} \vw_a^\top \vx' - \vw_A \vx$ \\
    $\vw_A = \vw_A + \alpha\,\delta\,\vx$ \\
    $\bar{R} = \bar{R} + \eta\,\alpha\,\delta$ \\
    $\vx = \vx'$ \\
 }
 \caption{Linear Q-learning with value-based reward centering}
 \label{algo:CDiscQ_linear}
\end{algorithm}

\vspace{-2mm}
\begin{algorithm}[h]
% \DontPrintSemicolon
\SetAlgoLined
\KwIn{The behavior policy $b$ (e.g., $\epsilon$-greedy)}
\SetKwInput{AP}{Algorithm parameters}
\AP{discount factor $\gamma$, step-size parameters $\alpha, \eta$}
Initialize value network, target network; initialize $\bar{R}$ arbitrarily (e.g., to zero) \\
Obtain initial observation $\vx$ \\
 \For{all time steps}
 {
    Take action $A$ according to $b$, observe $R, \vx'$ \\
    Store tuple $(\vx, A, R, \vx')$ in the experience buffer \\
    \If{time to update estimates}
    {
        Sample a minibatch of transitions $(\vx, A, R, \vx')^b$ \\
        For every $i$-th transition: $\delta_i = R_i - \bar{R} + \gamma \max_{a} \hat{q}(\vx_i',a) - \hat{q}(\vx_i,A_i)$ \\
        Perform a semi-gradient update of the value-network parameters with the $\delta^2$ loss \\ %a loss function of $\delta^2$ \\
        $\bar{R} = \bar{R} + \eta\, \alpha\, \text{mean}(\delta)$ \\
        Update the target network occasionally \\
    }
    $\vx = \vx'$ \\
 }
 \caption{(Non-linear) DQN with value-based reward centering}
 \label{algo:CDiscQ_nonlinear}
\end{algorithm}

We recommend two small but useful optimizations to these general pseudocodes in Appendix \ref{app:more_results}.
The python code is available at \href{https://github.com/abhisheknaik96/continuing-rl-exps}{github.com/abhisheknaik96/continuing-rl-exps}.

%%%%%%%%%%%%%%%%%%%%%%%%%%%%%%%%%%%%%%%%%%%%%%%%

\section{Theoretical Details}
\label{app:theory}

This section presents (a) the complete convergence result for Q-learning with value-based centering using Devraj and Meyn's (2021) analysis, and (b) quantifies the reduction in constant state--action-independent offset in the value estimates.

Suppose the agent's interaction with the MDP follows a stationary behavior policy $b \in \Pi$. Let $S_t, A_t$ denote the state-action pair occurring at time step $t$, followed by the reward $R_{t+1}$ and next state $S_{t+1}$. 
Let $\nu_t(s, a)$ denote the number of times a state-action pair $(s, a)$ has occurred up to and including time step $t$. 
The update rules of Q-learning with value-based centering are:
\begin{align}
    Q_{t+1}(S_t, A_t) &\doteq Q_t(S_t, A_t) + \alpha_{\nu_t(S_t, A_t)}\, \delta_t, \label{eq:cdisq_update_val} \\
    \bar{R}_{t+1} &\doteq \bar{R}_t + \eta\, \alpha_{\nu_t(S_t, A_t)}\, \delta_t, \label{eq:cdisq_update_rbar} \\
    \text{where,} \qquad \delta_t &\doteq R_{t+1} - \bar{R}_t + \gamma \max_{a'} Q_t(S_{t+1}, a') - Q_t(S_t, A_t),\label{eq:cdisq_tderror}
\end{align}
$\eta > 0$, and $\alpha_n = c / (n+d)$ where $c, d > 0$ for all $n \geq 1$.\footnote{Devraj and Meyn (2021) considered the step-size sequence $1/n$ in their algorithm but it can be easily verified that $\alpha_n = c / (n+d)$ also satisfies the step-size condition required by Borkar and Meyn's (2000) seminal result (that was used by Devraj \& Meyn (2021) to show the convergence of their algorithm).}

% \begin{assumption}\label{assu: markov chain}
%     The joint process $\{S_t, A_t\}$ is an irreducible Markov chain, that is, starting from every state-action pair, there is a non-zero probability of transitioning to any other state-action pair in a finite number of steps.
% \end{assumption}
\vspace{3mm}
\textbf{Theorem 1.} (Formal)
% \begin{theorem} (Formal)
    \textit{If the joint process $\{S_t, A_t\}$ induced by the stationary behavior policy is an irreducible Markov chain, that is, starting from every state-action pair, there is a non-zero probability of transitioning to any other state-action pair in a finite number of steps, then $(Q_t, \bar R_t)$ in tabular Q-learning with value-based centering (\ref{eq:cdisq_update_val}–\ref{eq:cdisq_tderror}) converges to a solution of $(\bar{\vq}^\gamma, \bar{r})$ in} \eqref{eq:bellman_equation_centered_optimality_action}.
% \end{theorem}

\begin{proof}
We first show that Q-learning with value-based centering is a member of the large family of Devraj and Meyn's (2021) Relative Q-learning algorithms with particular choices of $\mu$ and $\kappa$. 
This allows us to utilize their convergence results. % such that $k$ is very close to $r(\pi_\gamma^*)$. 

The general Relative Q-learning algorithm updates its tabular estimates $\tilde{Q}^\gamma: \calS\times\calA\to\mathbb{R}$ at time step $t$ using $(S_t, A_t, R_{t+1}, S_{t+1})$ as (in our notation):
\begin{align}
    \tilde{Q}^\gamma_{t+1}(S_t, A_t) \doteq \tilde{Q}^\gamma_t(S_t, A_t) + \alpha_t \big[ R_{t+1} - f(\tilde{Q}^\gamma_t) + \gamma \max_{a'} \tilde{Q}^\gamma_t(S_{t+1}, a') - \tilde{Q}^\gamma_t(S_t, A_t) \big], \label{eq:update_relativeq}
\end{align}
where, $f(\tilde{Q}^\gamma_t) \doteq \kappa \sum_{s,a} \mu(s,a) \tilde{Q}^\gamma_t(s,a)$, $\kappa>0$ is a scalar, and $\mu: \calS \times \calA \to [0,1]$ is a probability mass function. 

Now note that updating both the average-reward and value estimates using the TD error (\ref{eq:cdisq_update_val} and \ref{eq:cdisq_update_rbar}) results in:
\begin{align*}
    \bar{R}_t - \bar{R}_0 = \eta \Big( \sum_{s,a} Q_t(s,a) - \sum_{s,a} Q_0(s,a) \Big).    
\end{align*}
To simplify the analysis, we can assume $\bar{R}_0=0$ and $\vQ_0= \vzero$ without loss of generality.
As a result, $\bar{R}_t = \eta \sum_{s,a} \tilde{Q}^\gamma_t(s,a)$.
We can then combine the updates (\ref{eq:update_cdiscq_action_values}–\ref{eq:update_cdiscq_rbar_TDerror}) in the tabular case to:
\begin{align}
    \tilde{Q}^\gamma_{t+1}(S_t, A_t) &\doteq \tilde{Q}^\gamma_t(S_t, A_t) + \alpha_t \big( R_{t+1} - \eta \sum_{s,a} \tilde{Q}^\gamma_t(s,a) + \max_{a'} \tilde{Q}^\gamma_t(S_{t+1}, a') - \tilde{Q}^\gamma_t(S_t, A_t) \big). \label{eq:update_cdiscq_combined}
\end{align}
Comparing \eqref{eq:update_relativeq} and \eqref{eq:update_cdiscq_combined}, we can see that Q-learning with value-based reward centering is an instance of Relative Q-learning with:
\begin{align*}
    \mu(s,a) = \frac{1}{|\calS||\calA|} \ \forall s,a, \quad \text{and} \quad \kappa = \eta {|\calS||\calA|}.
\end{align*}
Devraj and Meyn's (2021) convergence result then applies. 
That is, 
\begin{align}
    \tilde{\vQ}^\gamma_t \to \tilde{\vQ}^\gamma_\infty &\doteq \vq_*^\gamma - \frac{\kappa}{1 - \gamma + \kappa} \vmu^\top \vq_*^\gamma \mathbf{1} \nonumber\\
    & = \vq_{*}^\gamma - \frac{\eta}{1 - \gamma + \eta |\calS||\calA|} \sum_{s,a} q_{*}^\gamma(s,a) \vone.  \label{eq:cdiscq_q_infty}
\end{align}
Hence,
\begin{align}
    \bar{R}_t \to \bar R_\infty &\doteq \eta \sum_{s,a} q_{*}^\gamma(s,a) - \frac{\eta^2 |\calS||\calA|}{1 - \gamma + \eta |\calS||\calA|} \sum_{s,a} q_{*}^\gamma(s,a) \nonumber \\
    &= \frac{\eta(1-\gamma)}{1 - \gamma + \eta |\calS||\calA|} \sum_{s,a} q_{*}^\gamma(s,a). \label{eq:cdisq_barR_infty}
\end{align}

We will now verify that $(\tilde{\vQ}^\gamma_\infty, \bar{R}_\infty)$ satisfy the Bellman equations \eqref{eq:bellman_equation_centered_optimality_action}. 
Recall that the solutions of the Bellman equation are of the form $\big( \tilde{\vq}^\gamma_* + \frac{k}{1-\gamma}\vone, r(\pi^*_\gamma) - k \big)$.
Since $\tilde{\vq}^\gamma_* = \vq^\gamma_* - \frac{r(\pi^*_\gamma)}{1-\gamma}$, we can re-write the solution class in terms of the discounted value function: $\big( \vq^\gamma_* + \frac{(k - r(\pi^*_\gamma))}{1-\gamma}\vone, r(\pi^*_\gamma) - k \big)$, or $\big( \vq^\gamma_* - \frac{d}{1-\gamma}\vone, d \big)$. 
For $d=\frac{\eta(1-\gamma)}{1 - \gamma + \eta |\calS||\calA|} \sum_{s,a} q_{*}^\gamma(s,a)$, we can see that $(\tilde{\vQ}^\gamma_\infty, \bar{R}_\infty)$ is a solution tuple of the Bellman equations.
\end{proof}

We can now characterize how close $\bar{R}_\infty$ is to $r(\pi^*_\gamma)$.
In general the expression for $\bar{R}_\infty$ \eqref{eq:cdisq_barR_infty} is cryptic.
However, a special case can shed some light.
We know that the average of the discounted value function for a policy w.r.t.~that policy's steady-state distribution is: $\sum_{s,a} d_\pi(s,a) q_\pi^\gamma(s,a) = \frac{r(\pi)}{1-\gamma}$.
Now suppose the steady-state distribution over state--action pairs is constant—$1/(|\calS||\calA|), \forall s, a$.
For that policy, $\frac{1}{|\calS||\calA|} \sum_{s,a}\, q_\pi^\gamma(s,a) = \frac{r(\pi)}{1-\gamma}$.
Substituting this in \eqref{eq:cdisq_barR_infty}, we get:
\begin{align}
    \bar{R}_\infty = \frac{\eta |\calS||\calA|}{1 - \gamma + \eta |\calS||\calA|} r(\pi_\gamma^*). \label{eq:barR_infty_close_to_rpi}
\end{align}
We can see that $\bar{R}_\infty$ approaches the true reward rate from below when $\eta |\calS||\calA| >> 1-\gamma$, which can be true in many problems of interest that have large state (and action) spaces.
That being said, note that this insight comes from a special case.
More generally, the convergence point of $\bar{R}_\infty$ (and hence $\tilde{\vQ}^\gamma_\infty$) is hard to interpret, which is a shortcoming we wish to resolve in future work. 
However, \eqref{eq:barR_infty_close_to_rpi} can serve as a rule of thumb. 

% \blue{should show the centered values have mean zero}
We end this section with a property of the centered discounted values.
\vspace{1mm}
\begin{lemma}
    The centered discounted values $\tilde{\vv}_\pi^\gamma$ are on average zero when weighted by the on-policy distribution $\vd_\pi$ induced by the policy $\pi$:
    \begin{align}
        \vd_\pi^\top \tilde{\vv}_\pi^\gamma = 0.
    \end{align}
\end{lemma}
\begin{proof}
    The proof is trivial after using the property that $\vd_\pi^\top \vv_\pi^\gamma = r(\pi)/(1-\gamma)$ (see Sutton \& Barto's (2018) Section 10.4 or Singh et al.'s (1994) Section 5.3). Since $\tilde{\vv}_\pi^\gamma = \vv_\pi^\gamma - r(\pi)/(1-\gamma)\vone$ (from \eqref{eq:def_centered_state_value_function}), $\vd_\pi^\top \tilde{\vv}_\pi^\gamma = 0$.
\end{proof}
%%%%%%%%%%%%%%%%%%%%%%%%%%%%%%%%%%%%%%%%%%%%%%%%
\newpage
\section{Experimental Details}
\label{app:more_results}

\textbf{Prediction}
`TD-learning with rewards centered by an oracle' refers to a version of TD-learning with centering in which the average-reward estimate is fixed to the (somehow) known average reward of the target policy.
In other words, the true average reward is known from the beginning and is subtracted from the observed rewards at each time step.
This algorithm is a good baseline because its rate of learning is likely the theoretical best among all TD-based prediction algorithms (in stationary problems where the average reward of the fixed target policy does not change with time).

Each algorithm was run on the random-walk problem for 50,000 steps and repeated 50 times each.
The step size $\alpha$ was decayed by $0.99999$ at each step. 
The values estimates for all variants and the average-reward estimate for TD-with-centering were initialized to zero.

We tested $\alpha\in\{0.01, 0.02, 0.04, 0.08, 0.16, 0.32\}$ and picked the one which resulted in the lowest average RMSVE across the training period for standard uncentered approach ($\alpha=0.04$ for $\gamma=0.9$ and $\alpha=0.08$ for $\gamma=0.99$). 
Corresponding to these step sizes, we tested the centering approaches' parameter $\eta$ within a coarse range of $\{1/640, 1/160, 1/40, 1/10\}$ and picked one based on the aforementioned criteria.
As mentioned earlier, this does not result in the best choice of $\alpha, \eta$ for the centering approaches, which is okay; we made sure the baselines are tuned appropriately.
% \blue{this experiment is not in the main text right now, but I can add it to the appendix}

%%%%%%%%%%%%%%%%%%%%%%%%%%%

\textbf{Control}
% In this section we provide the remaining experimental details and additional results that supplement the ones in the main text.
Table \ref{tab:list_of_hyperparameters} contains a list of all the hyperparameters tested that are common across all the domains: $\gamma, \alpha, \eta$. 
Note that setting $\eta=0$ and initializing the average-reward estimate to zero, Q-learning with reward centering behaves exactly like standard Q-learning. 
% The number of timesteps, number of runs, initializations are reported in the main text. 
For each set of parameters, the algorithms were run for $N$ steps and repeated $R$ times.
The $(N, R)$ tuples for each problem were: Access-Control Queuing: $(80k, 50)$; PuckWorld: $(300k, 20)$, Pendulum: $(100k, 15)$; Catch (linear): $(20k, 50)$; Catch (non-linear): $(80k, 15)$. 
For generating variants of the problems, we shifted the rewards by a range of numbers roughly proportional to the scale of rewards in the original problem: Access-Control Queuing and PuckWorld: \{-8, -4, 0, 4, 8\}; Pendulum: \{-12, -6, 0, 6, 12\}; Catch: \{-4, -2, 0, 2, 4\}.

The agent's behavior policy was always $\epsilon$-greedy with fixed $\epsilon=0.1$.
For all the experiments, the average-reward estimate was initialized to zero. 
The value-estimation weights were initialized to zero in the tabular and linear experiments; the weights were initialized to small values around zero in the non-linear experiments (the default initialization in PyTorch (Paszke et al., 2019)).
For the linear experiments we used 16 tiles of size $4\times4\times4$ for Catch and 32 tiles of size $4\times4\times4\times4\times4\times4$ for PuckWorld. 
These numbers and sizes were not specifically optimized for any problem or algorithm. 

\begin{table}[b]
\centering
\caption{List of hyperparameters tested for each domain}
\label{tab:list_of_hyperparameters}
\begin{tabular}{@{}lccc@{}}
\toprule
\textbf{} & $\gamma$ & $\alpha$ & $\eta$ \\ \midrule
\multicolumn{1}{c}{\begin{tabular}[c]{@{}c@{}}\textbf{Access-Control Queuing}\\(tabular)\\[2mm]\end{tabular}} & \begin{tabular}[c]{@{}c@{}}{[}0.5, 0.8, 0.9,\\ 0.99, 0.999, 1{]}\\[2mm]\end{tabular} & \begin{tabular}[c]{@{}c@{}}{[}1/128, 1/64, 1/32, \\ 1/16, 1/8, 1/4, 1/2, 1{]}\\[2mm]\end{tabular} & \begin{tabular}[c]{@{}c@{}}{[}0, 1/256, 1/64, \\ 1/16, 1/4, 1{]}\\[2mm]\end{tabular} \\
\multicolumn{1}{c}{\begin{tabular}[c]{@{}c@{}}\textbf{PuckWorld}\\ (linear)\\[2mm]\end{tabular}} & \begin{tabular}[c]{@{}c@{}}{[}0.5, 0.8, 0.9,\\ 0.99, 0.999, 1{]}\end{tabular} & \begin{tabular}[c]{@{}c@{}}{[}0.01, 0.1, 0.3, 0.5, \\ 0.7, 0.9, 1.0, 1.1{]}\end{tabular} & \begin{tabular}[c]{@{}c@{}}{[}0, 1/256, 1/64, \\ 1/16, 1/4, 1{]}\end{tabular} \\
\multicolumn{1}{c}{\begin{tabular}[c]{@{}c@{}}\textbf{Catch}\\ (linear)\\[2mm]\end{tabular}} & \begin{tabular}[c]{@{}c@{}}{[}0.5, 0.8, 0.9,\\ 0.99, 0.999, 1{]}\end{tabular} & \begin{tabular}[c]{@{}c@{}}{[}1/128, 1/64, 1/32, \\ 1/16, 1/8, 1/4, 1/2, 1{]}\end{tabular} & \begin{tabular}[c]{@{}c@{}}{[}0, 1/256, 1/64, \\ 1/16, 1/4, 1{]}\end{tabular} \\
\multicolumn{1}{c}{\begin{tabular}[c]{@{}c@{}}\textbf{Catch}\\(non-linear)\\[2mm]\end{tabular}} & \begin{tabular}[c]{@{}c@{}}{[}0.5, 0.8, 0.9,\\  0.99, 0.999, 1{]}\end{tabular} & \begin{tabular}[c]{@{}c@{}}{[}1/512, 1/256, 1/128, \\ 1/64, 1/32, 1/16, 1/8{]}\end{tabular} & \begin{tabular}[c]{@{}c@{}}{[}0, 1, 2, \\ 4, 8, 16{]}\end{tabular} \\
\multicolumn{1}{c}{\begin{tabular}[c]{@{}c@{}}\textbf{Pendulum}\\(non-linear)\\[2mm]\end{tabular}} & \begin{tabular}[c]{@{}c@{}}{[}0.5, 0.8, 0.9, \\ 0.99, 0.999, 1{]}\end{tabular} & \begin{tabular}[c]{@{}c@{}}{[}1/512, 1/256, 1/128,\\ 1/64, 1/32, 1/16, 1/8{]}\end{tabular} & \begin{tabular}[c]{@{}c@{}}{[}0, 1, 2, \\ 4, 8, 16{]}\end{tabular} \\ \bottomrule
\end{tabular}
\end{table}

We set commonly used values for the various parameters of the deep RL (non-linear) experiments: the batch size was 64, the value-network and reward-rate parameters were updated every 32 steps, the target network was updated every 128 steps, the experience buffer size was 10,000. Apart for the main step-size parameter, the default parameters (set by PyTorch) were used for the Adam optimizer (Kingma \& Ba, 2014).

Centering in the non-linear setting (that is, with DQN) in its current form requires a large value of $\eta$ compared to the the tabular or linear versions.
The reason is how a minibatch is used in the implementation of this deep RL algorithm.
In line 10 of Algorithm \ref{algo:CDiscQ_nonlinear}, the mean of the TD errors of the minibatch of transitions is taken. 
The mean can make the overall gradient for the reward-rate update very small, so a large value of $\eta$ can be used.

% The two algorithmic optimizations mentioned in Appendix \ref{app:pseudocodes} were used in all the experiments. 
In our implementations we added two simple optimizations:
\begin{enumerate}\itemsep0mm
    \item Make the average-reward estimate completely independent of its initialization: this can be done using the unbiased constant step-size trick (see Sutton \& Barto's (2018) Exercise 2.7). 
    \item Propagate the changes to the average-reward estimate faster: this can be done by first computing the TD error, then updating the reward-rate estimate, then recomputing the TD error with the new reward-rate estimate, and finally updating the value estimate(s).
\end{enumerate}
These optimizations did not affect the overall trends in the results but provided a small yet noticeable improvement for a tiny computational cost, hence we recommend using them. 

For the experiments involving a shift in the problem rewards, the rewards obtained on each problem variant are not directly comparable. 
For intuition, imagine the first four rewards in the original problem be 2,0,3,1. 
In a variant of the problem with 5 added to all the rewards, the first four rewards may now appear to be 7,5,8,4.
An agent solving the latter problem might trivially appear better than one solving the former problem even though its fourth reward was relatively lower. 
To compare them meaningfully, from the rewards obtained by an agent, we can subtract the constant that was added in the first place to all the problem's rewards.
That is, we can shift the rewards back to make fair comparisons across problem variants.
This is what we did when presenting the results of the shifting experiments; this is explicitly denoted by the word ``shifted'' in the $y$-axis label.

Figures \ref{fig:results_accesscontrol_sensitivity_gamma0.9}–\ref{fig:results_pen_sensitivity_gamma0.8} supplement the main trends shown in the main text: the effectiveness of centering increases as the discount factor approaches 1; with reward centering, the algorithms are more robust to any constant shifts in the rewards; the performance of reward centering is quite robust to the choice of the parameter $\eta$.

\begin{figure}[h]
    \centering
    \includegraphics[width=0.95\textwidth]{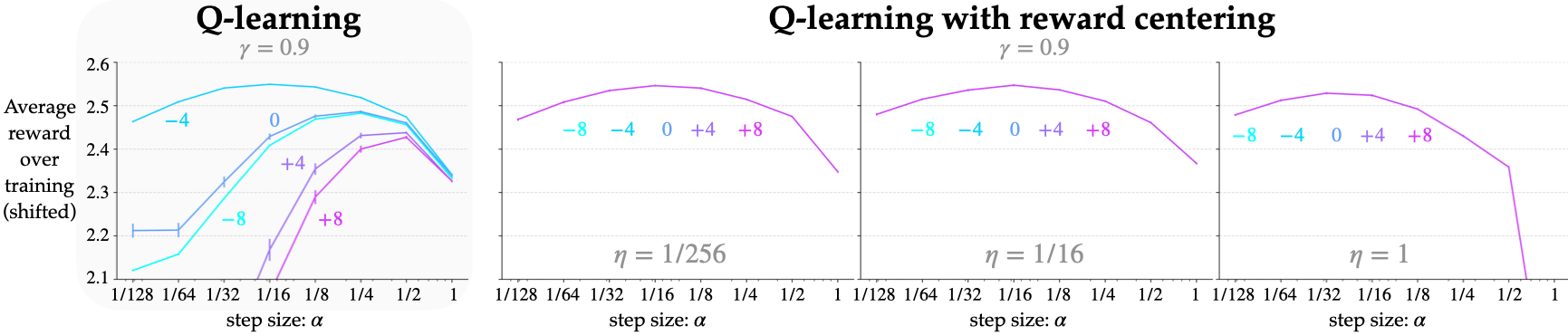}
    \caption{Parameter studies showing the sensitivity of the two algorithms’ performance on variants of the Access-Control Queuing domain.
    The error bars indicate one standard error, which at times is less than the width of the lines. 
    \textit{Far left:} Without centering, the performance of Q-learning differed significantly on the different variants over a broad range of the step-size parameter $\alpha$. 
    \textit{Center to right:} With centering, the performance was about the same across the problem variants, and was quite robust to the choice of its parameter $\eta$. 
    All the curves correspond to $\gamma=0.9$; the trends were consistent across other discount factors.
    }    \label{fig:results_accesscontrol_sensitivity_gamma0.9}
\end{figure}

\begin{figure}
    \centering
    \includegraphics[width=0.8\textwidth]{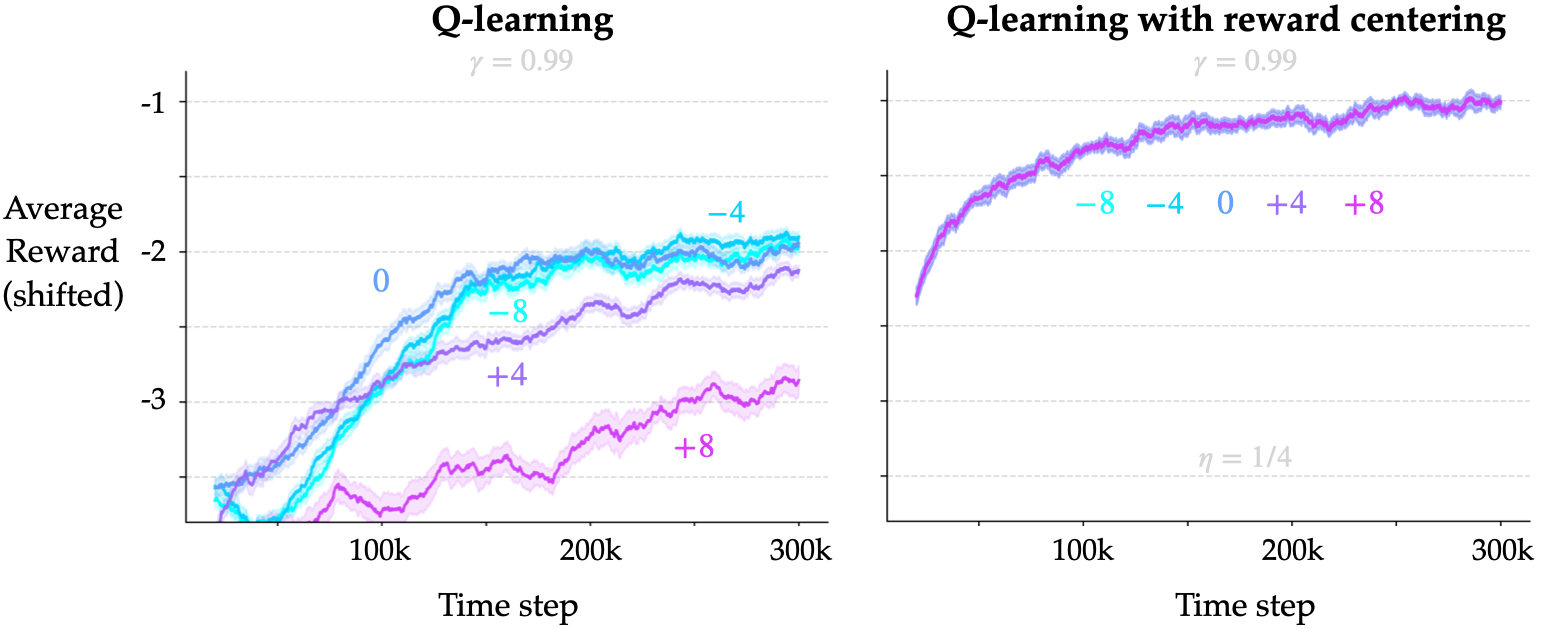}
    \caption{Learning curves for Q-learning with and without centering on variants of the PuckWorld problem when $\gamma=0.99$. 
    The performance without centering was different on each variant while that with centering was roughly the same. Reward centering also resulted in much faster learning. 
    These trends were consistent across values of $\gamma$.}
    \label{fig:results_PW_learningcurves_gamma0.99}
\end{figure}

\begin{figure}
    \centering
    \includegraphics[width=0.95\textwidth]{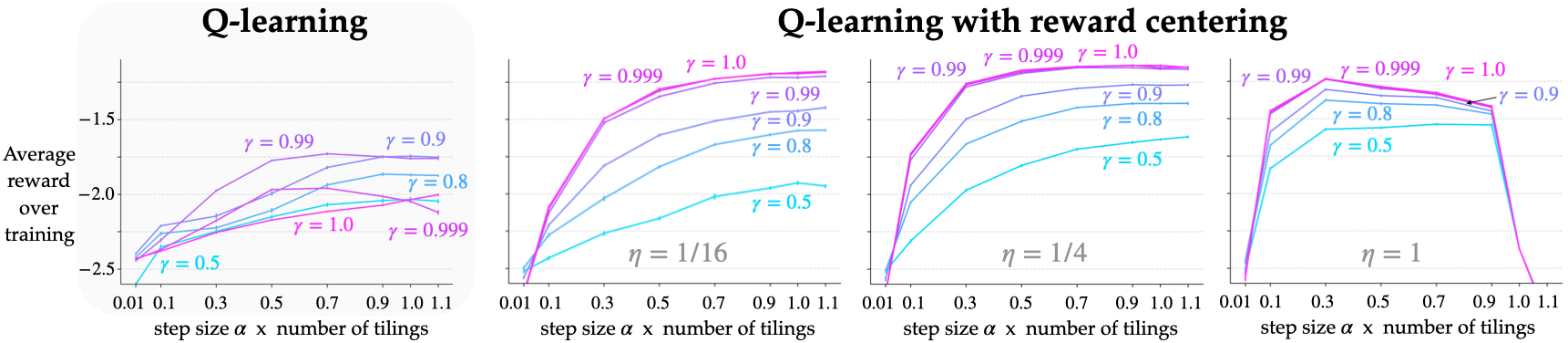}
    \caption{Parameter studies showing the sensitivity of the algorithms’ performance to their parameters on the PuckWorld domain. \textit{Far left:} Without centering, Q-learning's performance was relatively poor for a large range of $\alpha$. \textit{Center to right:} For each discount factor, the performance of Q-learning with centering was better across a broad range of $\alpha$. 
    % Additionally, the performance only changed a little w.r.t.~$\eta$.
    }
    \label{fig:results_PW_sensitivity_offset0}
\end{figure}

\begin{figure}
    \centering
    \includegraphics[width=0.95\textwidth]{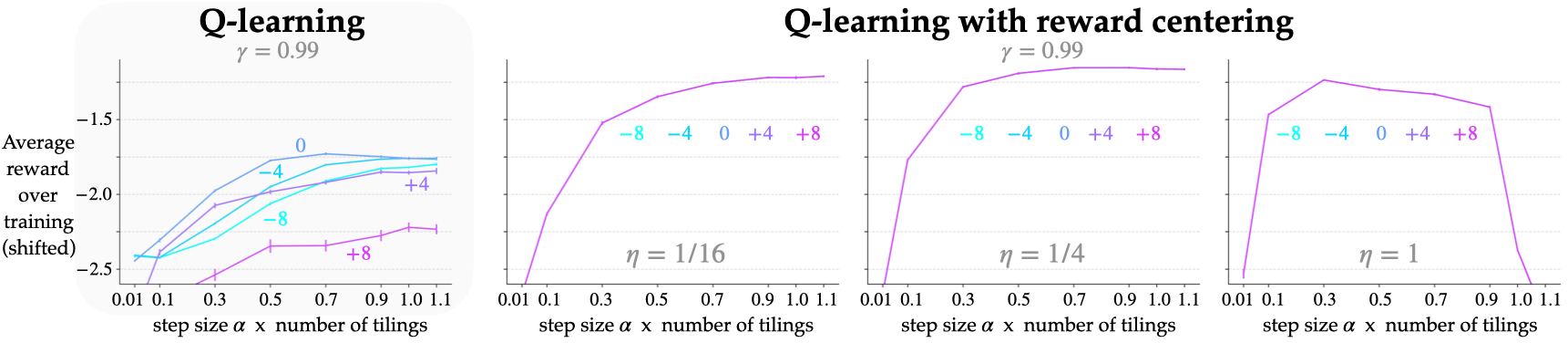}
    \caption{Parameter studies showing the sensitivity of the algorithms’ performance to variants of the PuckWorld domain. 
    The error bars indicate one standard error, which at times is less than the width of the lines. 
    \textit{Far left:} Without centering, the performance of Q-learning differed significantly on the different variants over a broad range of the step-size parameter $\alpha$. 
    \textit{Center to right:} With centering, the performance was about the same across the problem variants, and was quite robust to the choice of its parameter $\eta$. 
    All the curves correspond to $\gamma=0.99$; the trends were consistent across other discount factors.}
    \label{fig:results_PW_sensitivity_gamma0.99}
\end{figure}

\begin{figure}
    \centering
    \includegraphics[width=0.8\textwidth]{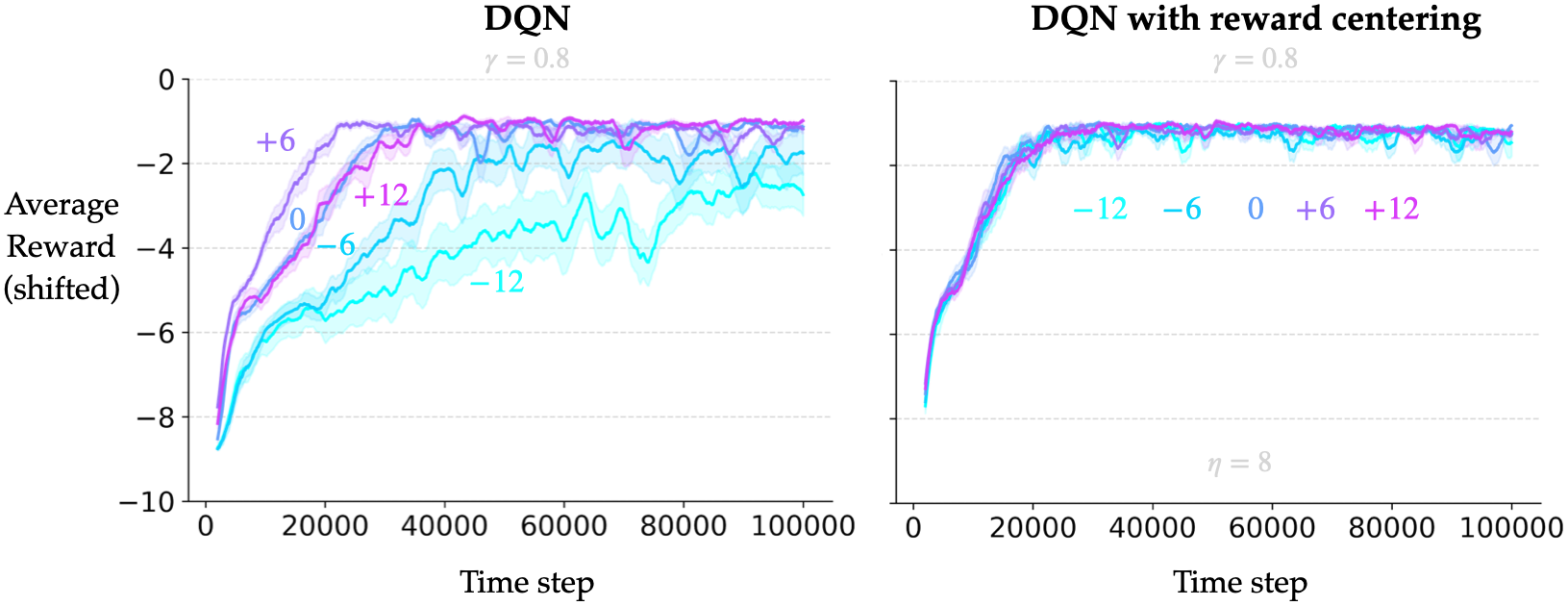}
    \caption{Learning curves for Q-learning with and without centering on variants of the Pendulum problem when $\gamma=0.8$. 
    The performance without centering was different on each variant while that with centering was roughly the same. Reward centering also resulted in much faster learning. 
    These trends were consistent across values of $\gamma$.}
    \label{fig:results_Pen_learningcurves_gamma0.99}
\end{figure}

\begin{figure}
    \centering
    \includegraphics[width=0.95\textwidth]{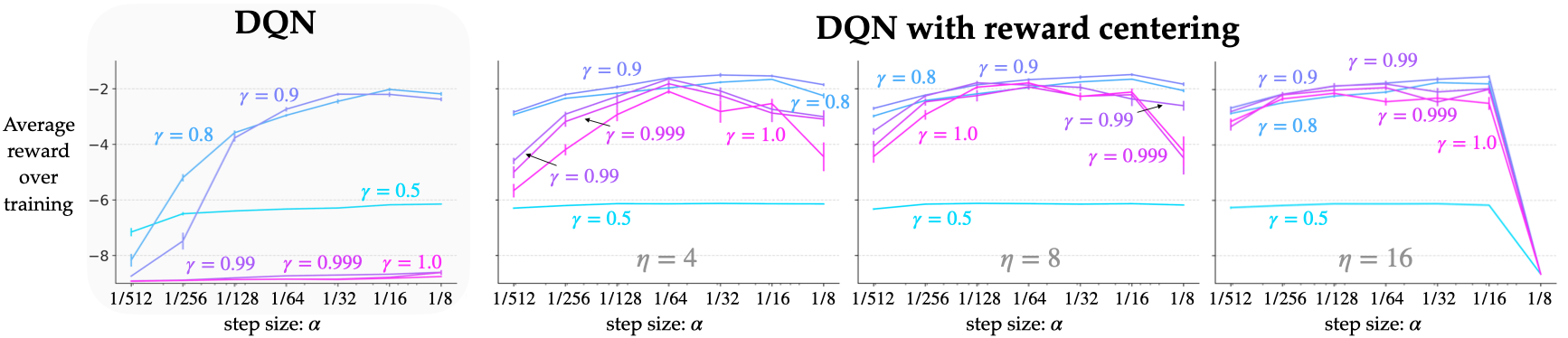}
    \caption{Parameter studies showing the sensitivity of the algorithms’ performance to their parameters on the Pendulum domain. $\gamma=0.5$ was too small to solve this problem. \textit{Far left:} The performance of DQN suffered for discount factors larger than 0.9. \textit{Center to right:} For each discount factor, the performance of DQN with centering was better across a broad range of $\alpha$. 
    Additionally, the performance was not too sensitive to the parameter $\eta$.}
    \label{fig:results_pen_sensitivity_offset0}
\end{figure}

\begin{figure}
    \centering
    \includegraphics[width=0.95\textwidth]{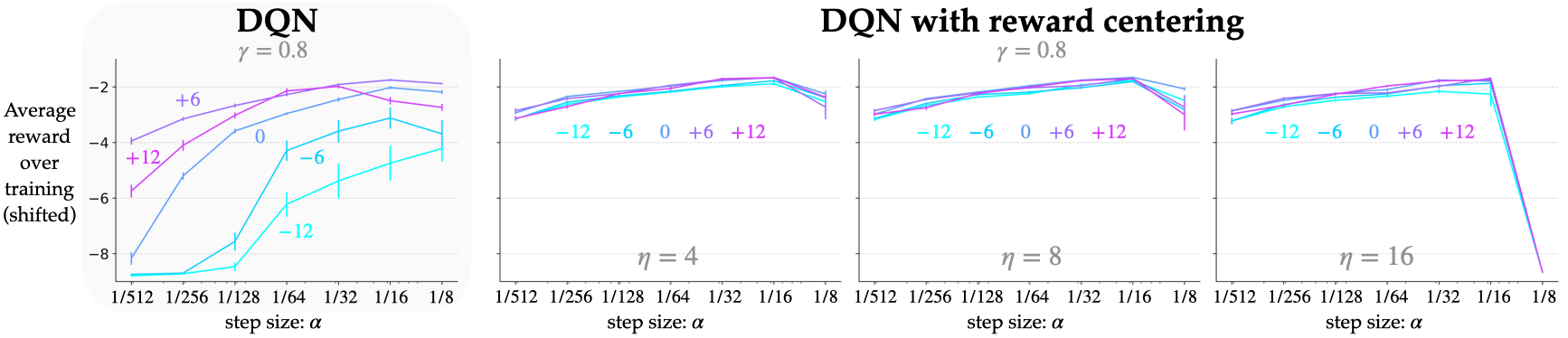}
    \caption{Parameter studies showing the sensitivity of the algorithms’ performance with $\gamma=0.8$ to variants of the Pendulum problem. \textit{Far left:} Without centering, the performance of DQN differed significantly on the different variants. \textit{Center to right:} With centering, the performance of DQN was about the same across the problem variants across a large range of the step size $\alpha$, and was also quite robust to the choice of $\eta$.}
    \label{fig:results_pen_sensitivity_gamma0.8}
\end{figure}

We also report preliminary results of PPO (Schulman et al., 2017) with and without centering. 
We chose to test these on the classic Mujoco problems (Todorov et al., 2012). 
% Several Mujoco domains are naturally continuing but are  episodic via timeouts. 
Mujoco domains are typically implemented as episodic problems; we converted them to continuing problems by (a) setting the episode-truncation parameter to a very large number, and (b) if applicable, resetting the domain to a starting state with a large negative reward if the agent enters an unrecoverable state.
We used value-based centering \eqref{eq:update_cdiscq_rbar_TDerror}, %at every policy update step, 
where $\delta_t$ corresponds to the advantage estimates computed by standard PPO. %and subsequently uses the centered reward for policy optimization. 

Figure \ref{fig:centering_results_ppo} shows the learning curves for PPO with and without centering. 
The $y$-axis shows the average reward obtained the agent over the last 1000 time steps. 
As with all the other experiments in this paper, the evaluation is online—there are no separate training or testing periods.
A careful study will take more time due to the large number of hyperparameters; in our preliminary experiments with 10 runs each, we found that centering results in a slight improvement on all the problems, with the most pronounced improvements on the Humanoid problem.
The step sizes corresponding to average-reward estimate for the different domains are: Hopper: 1E{-4}, HalfCheetah: 1E{-3}, Walker2D: 2E{-5}, Swimmer: 5E{-5}, Humanoid: 1E{-2}, Ant: 1E{-4}.

% In figure~\ref{fig:centering_results_ppo}, we see that centered PPO consistently results in improved reward rates across all six domains, with the most pronounced results seen in the Humanoid domain. 

\begin{figure}
    \centering
    \includegraphics[width=0.31\linewidth]{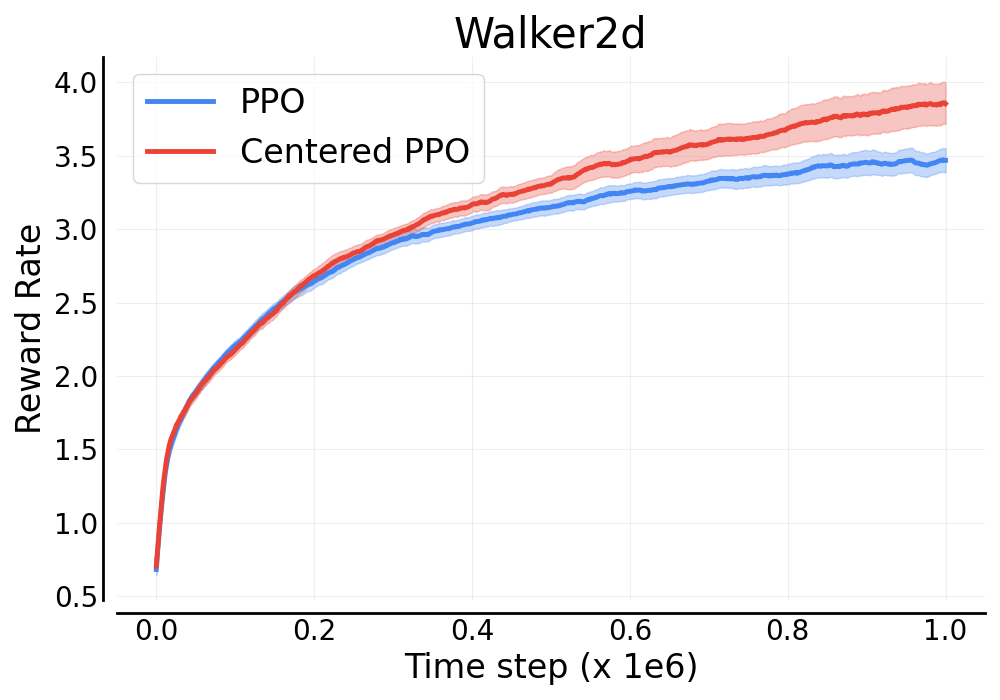}
    \includegraphics[width=0.31\linewidth]{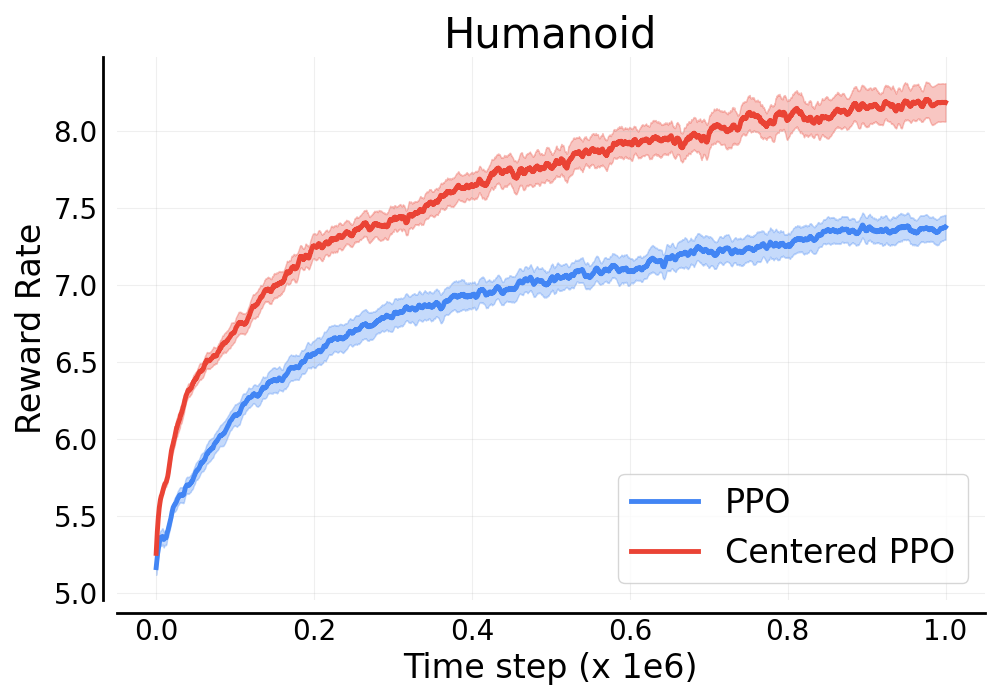}
    \includegraphics[width=0.31\linewidth]{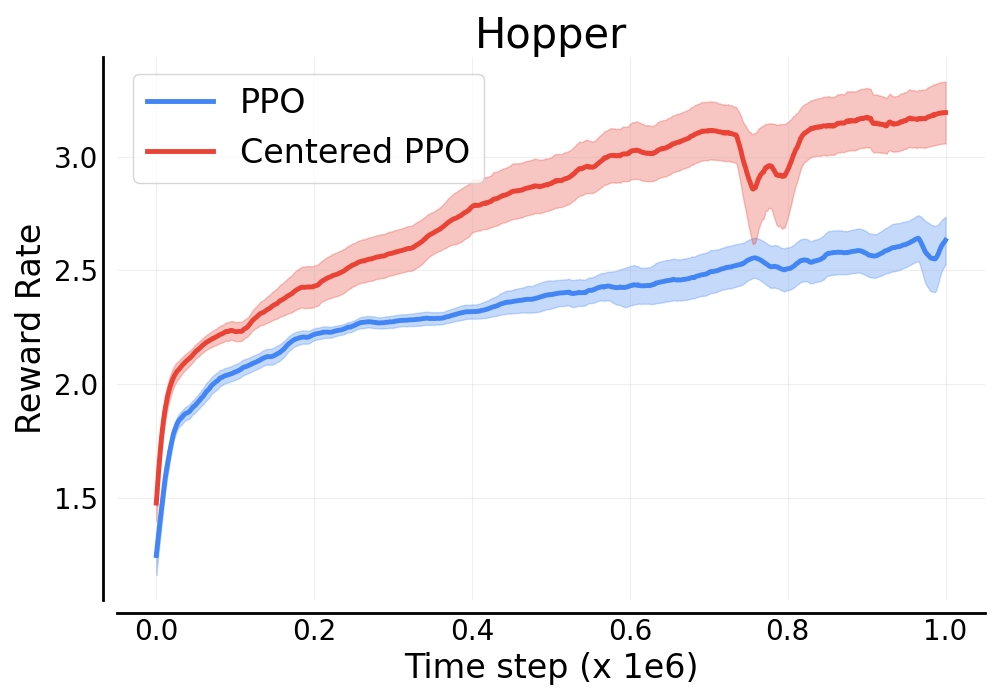} \\
    \includegraphics[width=0.31\linewidth]{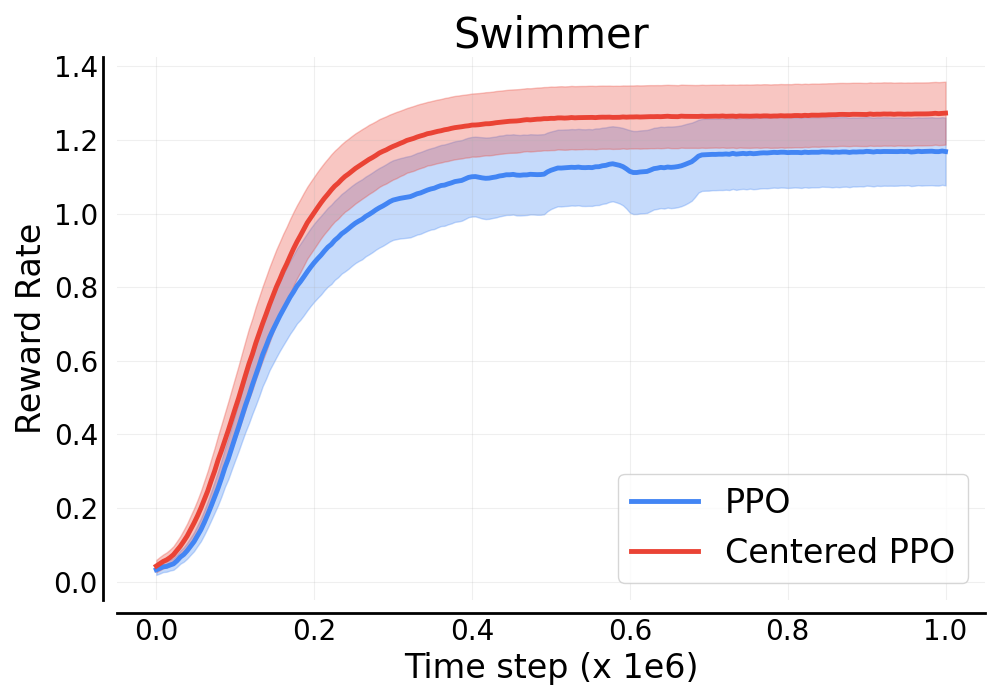}
    \includegraphics[width=0.31\linewidth]{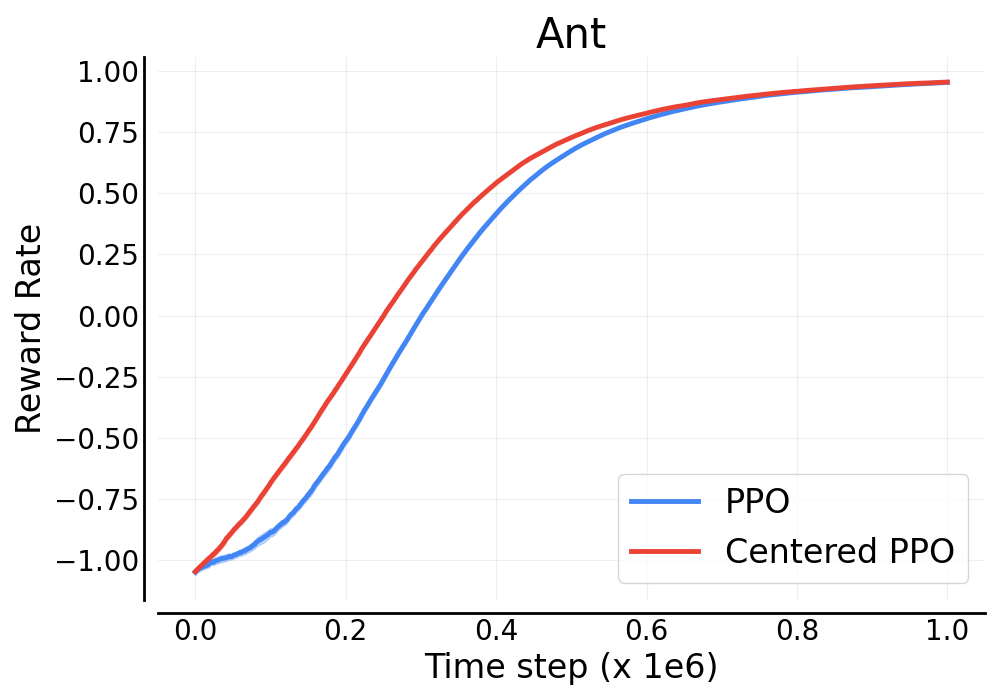}
    \includegraphics[width=0.31\linewidth]{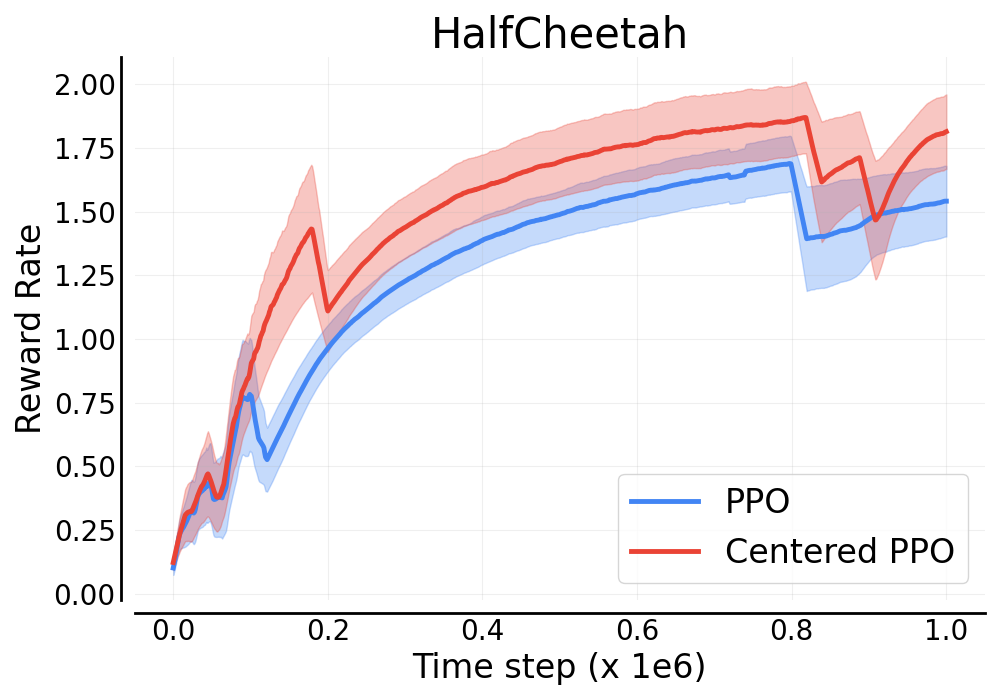}
    \caption{Learning curves for PPO with and without centering on continuing versions of six Mujoco domains. The solid lines and the shaded region denote the mean and one standard error over 10 independent runs.}
    \label{fig:centering_results_ppo}
\end{figure}

%%%%%%%%%%%%%%%%%%%%%%%%%%%%%%%%%%%%%%%%%%%%%%%%
% \newpage
\section{Connections to Related Approaches}
\label{app:related_work}

Concurrently with Devraj and Meyn (2021), Schneckenreither (2020) realized the Laurent series decomposition %\eqref{eq:laurent_series_expansion_action_values} 
suggests that an explicit estimate of the average reward can completely remove the offset.
So they proposed an algorithm which to estimate and subtract the average reward, with two important differences: (a) the average-reward estimate is updated only after non-exploratory actions, and (b) the algorithm has two discount factors to aim for the strongest optimality criterion—Blackwell optimality.
Schneckenreither did not provide any convergence result for their algorithm.
However, they analyzed that \textit{if} the algorithm converged to the desired fixed point, then the resulting policy would be (Blackwell-)optimal.
Wan et al.~(2021) pointed out the average-reward estimate can be updated at every time step, including ones with exploratory actions, and showed almost-sure convergence of their algorithms.
Combining those insights with Devraj and Meyn's, we show the convergence of Q-learning with value-based reward centering.

Reward centering and the advantage function have orthogonal benefits. 
The advantage function benefits the actor by reducing the variance of the updates in the policy space (Sutton \& Barto, 2018; Schulman et al., 2016).
On the other hand, reward centering benefits the critic's or baseline's estimation by eliminating the need to estimate the large state-independent constant offset. 
Both the quantities involved in the advantage function---${a}_\pi^\gamma(s,a) = {q}_\pi^\gamma(s,a) - {v}_\pi^\gamma(s), \forall s,a$---have the large state-independent offset $r(\pi)/(1-\gamma)$.
The net effect of the offset is zero when they are subtracted. 
But the key point is that both the state- and action-value estimates include the large offset. 
Reward centering removes the need to estimate the large offset for both the state- and action-value function, which simplifies the critic-estimation problem.
The actor update is left unchanged with reward centering because the advantage function itself remains unchanged: $\tilde{a}_\pi^\gamma(s,a) = \tilde{q}_\pi^\gamma(s,a) - \tilde{v}_\pi^\gamma(s)$, because $\tilde{q}_\pi^\gamma(s,a) = q_\pi^\gamma(s,a) - r(\pi)/(1-\gamma)$ and $\tilde{v}_\pi^\gamma(s) = v_\pi^\gamma(s) - r(\pi)/(1-\gamma)$. 
Hence, we expect reward centering to benefit all the algorithms that estimate values, which include all actor-critic methods that involve advantage estimation.
% An empirical study of centered variants of several common policy-based algorithms is a ripe avenue of future work.

Dividing all the rewards with a (potentially changing) scalar number is typically referred to as reward scaling (see, e.g., Engstrom et al., 2020).
Just like reward centering, reward scaling does not change the ordering of policies in a continuing problem. 
Scaling reduces the spread of the rewards, centering brings them close to zero, both of which can be favorable to complex function approximators such as artificial neural networks that are used for value estimation starting from a close-to-zero initialization. 
The popular stable\_baselines3 repository scales (and clips\footnote{Reward clipping in general changes the problem. Blinding the agent from large rewards can impose a performance ceiling or make some games impossible to solve (Schaul et al.’s (2021) Section 4.3 discusses this in the context of Atari problems).}) the rewards by a running estimate of the variance of the discounted returns (\href{https://github.com/DLR-RM/stable-baselines3/blob/e3dea4b2e03da6fb7ea70db89602909081a7967b/stable_baselines3/common/vec_env/vec_normalize.py#L256}{github.com/DLR-RM/stable-baselines3/blob/master/stable\_baselines3/common/vec\_env/vec\_normalize.py\#L256}). 
Mean-centering the rewards as well would be beneficial for continuing domains. 
Note that the mechanism of computing the mean and variance is more complicated in the off-policy setting than the on-policy setting. 
% We use the TD-error-based technique to estimate the mean, which is theoretically sound even in the off-policy case. 
Our TD-error-based technique is likely part of the final solution for the off-policy setting.
Simply maintaining a running estimate of the variance (as in the stable\_baselines’ approach) introduces a bias. 
As mentioned earlier, Schaul et al.’s (2021) technique is a good starting point. 

Reward centering can be seen as reward shaping (Ng et al., 1999) with a constant state-independent potential function: $\Phi(s) = r(\pi)/(1-\gamma), \forall s$. 
Their Theorem 1 then reiterates that reward centering does not change the optimal policy of the problem. 
A possible drawback of reward shaping is that fully specifying the potential-based shaping function can be tricky, especially for problems with large state spaces. 
In the case of reward centering this is relatively easy: the potential function is constant across the entire state space, and we know how to learn the average reward reliably from data.

Finally, we note that the idea of shifting rewards has been explored in episodic problems. Sun et al.'s (2022) experiments show that subtracting a suitable constant from all the rewards can help in some episodic problems. 
However, we do not expect shifting or centering to help in general in episodic problems: 
shifting all the rewards by a constant does not change a continuing problem, but can change episodic problems (e.g., the gridworld example from the final section of the main text).

%%%%%%%%%%%%%%%%%%%%%%%%%%%%%%%%%%%%%%%%%%%%%%%%

\end{document}